\documentclass{article}

\usepackage[preprint]{neurips_2020}

\usepackage[utf8]{inputenc} 
\usepackage[T1]{fontenc}    
\usepackage{booktabs}       
\usepackage{amsfonts}       
\usepackage{nicefrac}       
\usepackage{microtype}      

\usepackage{caption}
\usepackage{graphicx}
\usepackage{amsmath}
\usepackage{amssymb}
\usepackage{algorithm}
\usepackage{algorithmic}
\usepackage{bm}
\usepackage{threeparttable}
\usepackage{mathrsfs}
\usepackage{lipsum}
\usepackage{amsthm} 
\usepackage{enumitem}
\usepackage{multirow, subfigure}
\usepackage{appendix}
\usepackage{hyperref}       
\usepackage{color}

\def\x{\bm{x}}
\def\ie{$i.e.$}
\def\eg{$e.g.$}

\long\def\comment#1{}

\newtheorem{theorem}{Theorem}
\newtheorem{lemma}{Lemma}
\newtheorem{defn}{Definition}

\title{Toward Adversarial Robustness via Semi-supervised Robust Training}

\author{%
	Yiming Li$^{1}$, Baoyuan Wu$^{2}$, Yan Feng$^{1}$
Yanbo Fan$^2$, Yong Jiang$^1$, Zhifeng Li$^2$, Shutao Xia$^1$\\
$^1$Tsinghua Shenzhen International Graduate School, Tsinghua University, China\\
$^2$Tencent AI Lab, China\\
\texttt{wubaoyuan1987@gmail.com}; \texttt{xiast@sz.tsinghua.edu.cn}\\
}

\begin{document}

\maketitle

\begin{abstract}
Adversarial examples have been shown to be the severe threat to deep neural networks (DNNs). One of the most effective adversarial defense methods is adversarial training (AT) through minimizing the adversarial risk $R_{adv}$,  which encourages both the benign example $\x$ and its adversarially perturbed neighborhoods within the $\ell_{p}$-ball to be predicted as the ground-truth label. In this work, we propose a novel defense method, the robust training (RT), by jointly minimizing two separated risks ($R_{stand}$ and $R_{rob}$), which are with respect to the benign example and its neighborhoods, respectively. The motivation is to explicitly and jointly enhance the accuracy and the adversarial robustness. We prove that $R_{adv}$ is upper-bounded by $R_{stand} + R_{rob}$, which implies that RT has similar effect as AT. Intuitively, minimizing the standard risk enforces the benign example to be correctly predicted, and the robust risk minimization encourages the predictions of the neighbor examples to be consistent with the prediction of the benign example.  Besides, since $R_{rob}$ is independent of the ground-truth label, RT is naturally extended to the semi-supervised mode (\ie, SRT), to further enhance the adversarial robustness. Moreover, we extend the $\ell_{p}$-bounded neighborhood to a general case, which covers different types of perturbations, such as the pixel-wise (\ie, $\x + \boldsymbol{\delta}$) or the spatial perturbation (\ie, $\bm{A} \x + \bm{b}$). Extensive experiments on benchmark datasets not only verify the superiority of the proposed SRT method to state-of-the-art methods for defensing pixel-wise or spatial perturbations separately, but also demonstrate its robustness to both perturbations simultaneously. The code for reproducing main results is available at \url{https://github.com/THUYimingLi/Semi-supervised_Robust_Training}.
\end{abstract}

\vspace{-1.5em}
\section{Introduction}
\vspace{-0.5em}

It has been shown that deep neural networks (DNNs) are vulnerable to adversarial examples \cite{szegedy2013,zhou2018,dong2019}. Considering the wide application of DNNs in many mission-critical tasks (\eg, face recognition), it is urgent to develop effective defense methods for adversarial examples. One of the most promising defense methods is adversarial training (AT) \cite{goodfellow2014, madry2017}, which minimizes the adversarial risk $R_{adv}$. It firstly defines the perturbed neighborhood set bounded by $\ell_p$-norm around each benign example $\x$, and the perturbation is generated by some off-the-shelf adversarial attack methods, such as PGD \cite{madry2017} or FGSM \cite{goodfellow2014}.  $R_{adv}$ indicates the maximal loss value among the neighborhood samples and the benign example. It couples the classification risks on both the benign example $\x$ and its surrounding $\ell_p$-bounded perturbed neighborhood examples. Minimizing $R_{adv}$ encourages both $\x$ and its neighborhoods to be correctly predicted. 

In this work, we propose to separate the classification risk on $\x$ and its perturbed neighborhoods, respectively. We propose a novel robust training (RT) method by minimizing $R_{stand} + \lambda R_{rob}$. In RT, $R_{stand}$ indicates whether benign example $\x$ is correctly predicted, while the robust risk $R_{rob}$ indicates the maximal value of the $0/1$-loss that measures whether $\x$ and one of its perturbed neighborhoods are predicted as the same class, among all neighborhoods of $\x$. We prove that $R_{adv}$ is upper-bounded by $R_{stand} + R_{rob}$, which guarantees that RT has the similar effect as AT. Specifically, minimizing $R_{rob}$ encourages the predictions of the neighborhood examples to be consistent with the prediction of $\x$. Compared to AT, RT has two additional benefits. First, the balance between the classification accuracy on benign examples and the robustness to adversarial examples can be explicitly controlled by adjusting a trade-off hyper-parameter, \ie, $\lambda$. 
Second, one important property of $R_{rob}$ is that the ground-truth label of $\x$ is not involved, while only the prediction of $\x$ is adopted. It means that $R_{rob}$ doesn't require labeled training examples. 
Therefore, RT can be naturally extended to the semi-supervised mode, dubbed semi-supervised robust training (SRT), by incorporating massive unlabeled examples into $R_{rob}$. Experiments verify that the usage of the unlabeled examples in SRT could significantly enhance both the accuracy and the robustness to adversarial examples. 

Moreover, we notice that most existing works on adversarial training adopt the $\ell_p$-bounded neighborhood when generating adversarial examples during the training. This neighborhood corresponds to the pixel-wise adversarial perturbations, \ie, $\x + \boldsymbol{\delta}$. 
Meanwhile, many researches have found that there are also many other types of adversarial perturbations, such as the spatial perturbation or the functional perturbation. 
It has been observed \cite{engstrom2019} that the defense designed for pixel-wise perturbations doesn't work for other types of perturbations. Is it possible to train a robust model to defend different types of adversarial perturbations simultaneously? To explore this problem, we extend the $\ell_p$-bounded neighborhood definition to a general case, since the perturbation type is closely related to the perturbed neighborhood. This new definition allows a general transformation-based perturbation and measures the distance between $\x$ and its perturbations using different metrics. This extension could be naturally embedded into SRT, as it only updates the form of neighborhood set. 
Consequently, the model trained using SRT with general perturbed neighborhood could simultaneously defend different types of perturbations.

The main contributions of this work are three-fold.
{\bf 1)} We propose a robust training method by jointly minimizing the standard risk and the robust risk, which is naturally extended the semi-supervised mode. 
{\bf 2)} By generalizing the definition of the perturbed neighborhood to cover different types of perturbations, our robust training achieves the joint robustness to different perturbations, such as the pixel-wise and spatial perturbation. 
{\bf 3)} Extensive experiments on benchmark datasets verify the superiority of the proposed SRT method to state-of-the-art adversarial training methods, as well as the robustness of SRT to pixel-wise and spatial perturbations simultaneously.

\vspace{-0.8em}
\section{Related Work}
\vspace{-0.8em}

\noindent \textbf{Supervised Adversarial Defense}.
DNNs are known to be vulnerable to different types of well-designed small adversarial perturbation, such as pixel-wise perturbation \cite{szegedy2013,poursaeed2018,zeng2019} and spatial perturbation \cite{fawzi2015,engstrom2019,yang2019}. Those attacks usually based on the relation between the prediction and the ground-truth label of the sample, \ie they are in a supervised manner. Toward the adversarial robustness against those attacks, several supervised adversarial defense methods were proposed. These methods can be roughly divided into three main categories, including adversarial training based defense \cite{goodfellow2014,madry2017,zhang2019}, detection based defense \cite{lu2017,metzen2017,liu2019}, and reconstruction-based defense \cite{gu2014,samangouei2018,bai2019}. Adversarial training based defense is currently the most mainstream research direction, which improves the adversarial robustness via adding various adversarially manipulated samples during the training process. For example, standard adversarial training improved the model robustness through training on adversarial examples generated by FGSM \cite{goodfellow2014} and PGD \cite{madry2017}, and the trade-off inspired adversarial defense (TRADES) generated manipulated samples by further considering the trade-off between robustness and accuracy \cite{zhang2019}. Recently, inspired by the geometric property of loss surface, other regularization-based defenses were also proposed \cite{pmlr-v97-simon-gabriel19a,qin2019adversarial,xu2019adversarial}. Although many defenses have been proposed, there is still a long way to go for solving the adversarial vulnerability problem.

Recently, some attempts have been proposed to defend multiple perturbations together. In \cite{engstrom2019}, they observed that pixel-wise defense techniques have limited benefit to the spatial robustness. This problem is further investigated in \cite{tramer2019}, where the author proved that there exists a trade-off in robustness to different types of attacks in a natural and simple statistical setting. In other words, it is likely that there is no defense that could reach best robustness under every attacks. How to defend against different types of attacks simultaneously is still an important open question.

\noindent \textbf{Semi-supervised Learning and its Usage in Defense}.
The semi-supervised learning focuses on how to utilize a large amount of unlabeled data to learn better models \cite{narayanan2010,rifai2011,sheikhpour2017}. Some recent works proposed to extend the supervised adversarial training to the semi-supervised setting. For example, the virtual adversarial loss \cite{miyato2018} is defined as the robustness of the conditional label distribution around each benign example against local perturbation, which firstly connected the semi-supervised learning with the adversarial training. 
Besides, some works uncovered the role of unlabeled data in the adversarial defense theoretically \cite{carmon2019,stanforth2019}. It is proved \cite{stanforth2019} that the sample complexity for learning a pixel-wise adversarially robust model from unlabeled data matches the fully supervised case up to constant factors, under simplified statistical settings. Similarly, \cite{carmon2019} proved that the gap of sample complexity between standard and robust classification can be filled by using additional unlabeled data in the Gaussian model. However, although they all provided some theoretical insights for the importance of unlabeled data in adversarial defense, the intrinsic reason is still not well answered.

\section{Proposed Method}

\subsection{Preliminaries}
We denote the classifier as $f_{\boldsymbol{w}}: \mathcal{X} \rightarrow [0,1]^{|\mathcal{Y}|}$, with $\mathcal{X} \subset \mathbb{R}^d$ being the instance space and $\mathcal{Y}= \{1,2,\cdots, K\}$ being the label space. $f(\x)$ indicates the posterior vector with respect to $K$ classes, and $C(\x) = \arg\max f_{\boldsymbol{w}}(\x)$ denotes the predicted label. 
The labeled dataset is denoted as $\mathcal{D}_L = \left\{(\bm{x}_i, y_i) | i = 1,\ldots, N_l \right\}$, where $(\bm{x}_i, y_i)$ is independent and identically sampled from an unknown latent distribution $\mathcal{P}_{\mathcal{X} \times \mathcal{Y}}$. 
Let $\mathcal{D}_L' = \left\{\bm{x}|(\bm{x}, y) \in \mathcal{D}_L \right\}$ indicates the instance set of $\mathcal{D}_L$, $\mathcal{D}_U = \{ \x_i | i = 1, \ldots, N_u \}$ denotes the unlabeled dataset. 

\begin{defn}
The $\epsilon$-bounded transformation-based neighborhood set of the benign example $\x$ is defined as follows:
\begin{flalign}
\mathcal{N}_{\epsilon, \bm{T}}(\bm{x}) = \left\{\bm{T}(\bm{x};\theta)| \  dist\left(\bm{T}(\bm{x};\theta), \bm{x}\right) \leq \epsilon \right\},
\end{flalign}
where $\bm{T}(\cdot;\theta)$ indicates a parametric transformation, and $dist(\cdot, \cdot)$ is a given distance metric corresponding to $\bm{T}(\cdot;\theta)$. 
Non-negative hyper-parameter $\epsilon$ denotes the maximum perturbation size.
\label{label: definition of neighborhood}
\end{defn}

\vspace{-1em}
{\bf Remark}.
$\mathcal{N}_{\epsilon, \bm{T}}$ can reduce to some widely used neighborhood sets through different specifications of $\bm{T}(\cdot;\theta)$ and $dist(\cdot, \cdot)$.
For example, if we set $\bm{T}(\bm{x};\theta) = \bm{x} + \theta$ and $dist\left(\bm{T}(\bm{x};\theta), \bm{x}\right) = \|\bm{T}(\bm{x};\theta) - \bm{x}\|_\infty$, then $\mathcal{N}_{\epsilon, \bm{T}}$ becomes the $\ell_{\infty}$-bounded neighborhood set used in the pixel-wise adversarial attack and defense \cite{papernot2016limitations,madry2017,zhang2019}.
If we set $\bm{T}(\bm{x};\theta) = [\cos{\theta}, -\sin{\theta}; \sin{\theta}, \cos{\theta}] \x$ and 
$dist\left(\bm{T}(\bm{x};\theta), \bm{x}\right) = \theta$, then 
$\mathcal{N}_{\epsilon, \bm{T}}$ indicates the rotation-based neighborhood set used in the spatial adversarial attack and defense \cite{engstrom2019,yang2019}.
This flexibility is important for developing a method to defend multiple types of adversarial perturbations, which will be shown later.

Based on the general transformation-based neighborhood defined above, we can extend the traditional adversarial risk and the robust risk to the form as follows:

\begin{defn} [Standard, Adversarial, and Robust Risk]
\label{def: three risks}
\end{defn}
\begin{itemize}
    \item \emph{The standard risk $R_{stand}$ measures whether the prediction of $\x$ (\ie, $C(\x)$), is same with its ground-truth label $y$. Its definition with respect to a labeled dataset $\mathcal{D}$ is formulated as 
    \begin{flalign}
    R_{stand}(\mathcal{D}) = \mathbb{E}_{(\bm{x}, y) \sim  \mathcal{P}_{\mathcal{D}}} \left[ \mathbb{I}\{C(\bm{x}) \neq y\}\right],
    \end{flalign}
    where $\mathcal{P}_{\mathcal{D}}$ indicates the distribution behind $\mathcal{D}$. 
    $\mathbb{I}(a)$ denotes the indicator function: $\mathbb{I}(a) = 1$ if $a$ is true, otherwise $\mathbb{I}(a) = 0$.}
    \cite{zhang2019}
    \item \emph{The $\mathcal{N}_{\epsilon, \bm{T}}$-based adversarial risk with respect to $\mathcal{D}$ is defined as 
    \begin{flalign}
    \hspace{-1em} R_{adv}(\mathcal{D}) = \mathbb{E}_{(\bm{x}, y) \sim \mathcal{P}_{\mathcal{D}}} [ \max \limits_{\bm{x'} \in \mathcal{N}_{\epsilon, \bm{T}}(\bm{x})} \mathbb{I}\{C(\bm{x'}) \neq y\}].
    \end{flalign}
    }
    \item \emph{The $\mathcal{N}_{\epsilon, \bm{T}}$-based robust risk with respect to $\mathcal{D}$ is defined as 
    \begin{flalign}
    \hspace{-1.4em} R_{rob}(\mathcal{D}) = \mathbb{E}_{\bm{x} \sim \mathcal{P}_{\mathcal{D'}}} [ \max \limits_{\bm{x'} \in \mathcal{N}_{\epsilon, \bm{T}}(\bm{x})} \mathbb{I}\{C(\bm{x'}) \neq C(\bm{x})\} ].
    \end{flalign}
    }
\end{itemize}

{\bf Remark}. 
$R_{adv}$ considers the classification accuracy on both the benign example $\x$ and its adversarial perturbations within $\mathcal{N}_{\epsilon, \bm{T}}(\x)$. 
$R_{stand}$ is only related to the classification accuracy on $\x$. 
$R_{rob}$ doesn't involve the ground-truth label $y$, and it only corresponds to the consistency between the prediction of $\bm{x}$ and its adversarial perturbations, \ie, robustness. 
Their relations are shown in Lemma \ref{lemma: relations among risks}. Due to the space limit, the proof of Lemma \ref{lemma: relations among risks} will be presented in the {\bf Appendix (Section 1)}.

\begin{lemma}\label{lemma: relations among risks}
Standard Risk, adversarial risk and robust risk of a sample $\bm{x}$ are correlated. Specifically,
\begin{equation}
    R_{adv}(\bm{x}) = R_{stand}(\bm{x}) + \left(1 - R_{stand}(\bm{x})\right)R_{rob}(\bm{x}).  
\end{equation}
\end{lemma}

\subsection{Robust Training}
\label{sec: robust training}

Based on the Definition \ref{def: three risks}, we propose a novel robust training (RT)\footnote{We notice that the name ``robust training" has been used in some other works, such as in \cite{man2011new, shariati2013, han2018co, zugner2019}, etc. However, their objective functions are totally different with ours. For example,  the "robust training" in \cite{zugner2019} indicates a robust hinge loss in the object function. Here we re-define "robust training" based on the robust risk.} method, as follows
\begin{equation}
    \min_{\boldsymbol{w}} R_{stand}(\mathcal{D}_L) + \lambda \cdot R_{rob}(\mathcal{D}_L),
    \label{eq: robust training}
\end{equation}
where $\lambda>0$ is a trade-off parameter. 
The relationship between RT and the standard adversarial training (AT) \cite{madry2017} is analyzed in Theorem \ref{theorem: relation between at and rt} (whose proof is in the \textbf{Appendix (Section 1)}). RT is minimizing the upper bound of $R_{adv}$, if we set $\lambda=1$.
It tells that RT could have the similar effect with AT. However, due to the separation of the classification accuracy and the robustness in RT, there are two important advantages compared to AT. 
{\bf 1)} $\lambda$ can explicitly control the trade-off between the accuracy to benign examples and the robustness to adversarial examples, which will be verified in later experiments. $\lambda$ can be adjusted according to the user's demand. 
{\bf 2)} What is more important, since $R_{adv}$ doesn't utilize the ground-truth labels, it can be defined with respect to any unlabeled example. Thus, it is natural to extend RT to the semi-supervised mode, as shown in Section \ref{sec: SRT}.

\begin{theorem}
\label{theorem: relation between at and rt}
The relationship between the adversarial training and the robust training is as follows
\begin{equation}
    \min_{\boldsymbol{w}} R_{adv}(\mathcal{D}) \leq \min_{\boldsymbol{w}} \{R_{stand}(\mathcal{D}) + R_{rob}(\mathcal{D})\}.
    \label{eq: relation between at and rt}
\end{equation}
\end{theorem}

\subsection{Semi-supervised Robust Training}
\label{sec: SRT}

As stated above, $R_{rob}$ could be defined with respect to any unlabeled data. Thus, we propose the semi-supervised robust training (SRT), as follows
\begin{flalign}
\min_{\boldsymbol{w}} R_{stand}(\mathcal{D}_L) + \lambda \cdot R_{rob}(\mathcal{D}_L \cup \mathcal{D}_U). 
\label{eq: SRT}
\end{flalign}
Note that RT (\ie, Eq. (\ref{eq: robust training})) is a special case of Eq. (\ref{eq: SRT}), in the case that $\mathcal{D}_U = \emptyset$.

Due to the non-differentiability of the indicator function $\mathbb{I}(\cdot)$ used in $R_{stand}$ and $R_{rob}$, we resort to its approximated loss function, such as the cross-entropy. 
Besides, we replace the expectation with respect to $\mathcal{P}_{\mathcal{D}}$ in Eq. (\ref{eq: SRT}) by the sample mean with respect to $\mathcal{D}$. Then, the approximated objective function is formulated as follows:

\begin{equation}\label{eq: SRT approximation}
\min_{\boldsymbol{w}} \frac{1}{| \mathcal{D}_L |} \sum_{(\x, y) \in \mathcal{D}_L} \mathcal{L}\big(f_{\boldsymbol{w}}(\bm{x}), y\big) 
+ \frac{\lambda}{| \mathcal{D}_{L}' \cup \mathcal{D}_U |} \sum_{\x \in \mathcal{D}_{L}' \cup \mathcal{D}_U} \max_{\x' \in \mathcal{N}_{\epsilon, \bm{T}}(\x)} \mathcal{L}\big(f_{\boldsymbol{w}}(\bm{x'}), C(\bm{x}) \big),
\end{equation}
where $\mathcal{L}(\cdot)$ indicates the cross-entropy function. 

Similar to the optimization for adversarial training, the minimax problem (\ref{eq: SRT approximation}) can be solved by alternatively solving the inner-maximization and the outer-minimization sub-problem, as follows: 
\begin{itemize}[leftmargin=2em]
\item {\bf Inner-maximization}: given $\boldsymbol{w}$, $\forall \x \in \mathcal{D}_{L}' \cup \mathcal{D}_U$, we derive an perturbed example $\x'$ by
\begin{flalign}
\x' \leftarrow \underset{\x' \in \mathcal{N}_{\epsilon, \bm{T}}(\x)}{\arg\max} \mathcal{L}\big(f_{\boldsymbol{w}}(\bm{x'}), C(\bm{x}) \big).
\end{flalign}
%
\item {\bf Outer-minimization}: given all generated $\x'$, the parameter $\boldsymbol{w}$ is updated as follows:

\begin{equation}\label{eq: outer-opt}
\boldsymbol{w} \leftarrow \underset{\boldsymbol{w}}{\arg\min}  ~\frac{1}{| \mathcal{D}_L |} \sum_{(\x, y) \in \mathcal{D}_L} \mathcal{L}\big(f_{\boldsymbol{w}}(\bm{x}), y\big) 
+
\frac{\lambda}{| \mathcal{D}_{L}' \cup \mathcal{D}_U |} \sum_{\x \in \mathcal{D}_{L}' \cup \mathcal{D}_U} \mathcal{L}\big(f_{\boldsymbol{w}}(\bm{x'}), C(\bm{x}) \big).
\end{equation}
\end{itemize}

According to the specified $\mathcal{N}_{\epsilon, \bm{T}}(\x)$, the inner-maximization could be solved by different adversarial attacks, such as PGD in pixel-wise scenario \cite{madry2017} or Worst-of-$k$ in spatial scenario \cite{engstrom2019}. For outer-minimization, there is no close-form solution to this problem. Instead, we update $\boldsymbol{w}$ using back-propagation \cite{D1986Learning} with the stochastic gradient descent \cite{zhang2004}.

Besides, adversarial training ($e.g.$, AT \cite{madry2017}, TRADES \cite{zhang2019} and SRT) are often much more time-consuming than the standard training. To alleviate this problem, we further propose an accelerated version of SRT (dubbed fast SRT), which is described and evaluated in the \textbf{Appendix (Section 9)}.

\begin{table}[htbp]
\small
  \begin{minipage}{0.49\linewidth}
    \centering
    \caption{Performance of SRT and RT under different spatial settings on CIFAR-10 dataset.}
    \begin{tabular}{l|c|c|c}
    \hline
        & Clean & RandAdv & GridAdv  \\ \hline
    Standard & 80.63   & 8.82  & 0.09 \\ \hline
    RT ($\lambda = 0.15$) &  85.24   & 64.45    & 41.23  \\
    SRT ($\lambda = 0.15$) & $\bm{88.04}$   &  $\bm{78.03}$ & $\bm{62.97}$  \\ \hline
    RT ($\lambda = 0.20$) & 85.71   & 66.43 & 44.28  \\
    SRT ($\lambda = 0.20$) & $\bm{88.87}$  & $\bm{78.99}$  & $\bm{64.83}$  \\ \hline
    RT ($\lambda = 0.25$) & 85.59   & 67.95 & 45.93 \\
    SRT ($\lambda = 0.25$) & $\bm{88.40}$   & $\bm{78.39}$ & $\bm{64.15}$  \\ \hline
    RT ($\lambda = 0.30$) & 84.99   & 67.47    & 45.72 \\
    SRT ($\lambda = 0.30$) & $\bm{87.99}$   & $\bm{78.12}$ & $\bm{62.73}$  \\ \hline
    \end{tabular}
    \label{SRTbetter(spatial)}
  \end{minipage}
  \quad
  \begin{minipage}{0.47\linewidth}
    \centering
    \caption{Performance of SRT and RT under different pixel-wise settings on CIFAR-10 dataset.}
    \begin{tabular}{l|c|c|c}
    \hline
        & Clean & FGSM & PGD  \\ \hline
    Standard & 87.69   & 6.65        & 0  \\ \hline
    RT ($\lambda = 0.2$) & 83.24  & 48.94        & 31.91 \\
    SRT ($\lambda = 0.2$) & $\bm{83.83}$  & $\bm{51.39}$        & $\bm{34.94}$ \\ \hline
    RT ($\lambda = 0.4$) &  81.05   & 51.15        & 36.01 \\
    SRT ($\lambda = 0.4$) & $\bm{82.28}$   &  $\bm{56.06}$       & $\bm{41.84}$ \\ \hline
    RT ($\lambda = 0.6$) & 79.93   & 51.56        & 37.73 \\
    SRT ($\lambda = 0.6$) & $\bm{81.03}$  & $\bm{56.83}$        & $\bm{44.60}$ \\ \hline
    RT ($\lambda = 0.8$) & 78.37   & 51.97        & 37.91  \\
    SRT ($\lambda = 0.8$) & $\bm{80.23}$   & $\bm{58.14}$  & $\bm{47.24}$ \\ \hline
    \end{tabular}
    \label{SRTbetter(pixel)}
  \end{minipage}
 \vspace{-1em}
\end{table}

\section{Experiments}
Unless otherwise specified, the training set is divided into two parts, one of which contains 10,000 samples serves as the labeled dataset, and the remaining part is used as the unlabeled dataset in all experiments. In addition, in the semi-supervised setting, we use batch sizes proportional to the dataset size, \ie for a total batch size $m$, the number of labeled samples and unlabeled samples within a batch is $m \cdot \frac{|\mathcal{D}_L|}{|\mathcal{D}_L|+|\mathcal{D}_U|}$ and $m \cdot \frac{|\mathcal{D}_U|}{|\mathcal{D}_L|+|\mathcal{D}_U|}$, respectively. 
Similar with most previous works of adversarial training,
we also conduct experiments on CIFAR-10 \cite{krizhevsky2009} and MNIST \cite{lecun1998} dataset. Note that other large scale dataset like the popular ImageNet \cite{russakovsky2015imagenet} has been rarely adopted in adversarial training, due to the unbearable high computational cost. As reported in \cite{xie2019feature}, training a ResNet-101 \cite{he2016} model using the AT method \cite{madry2017} on ImageNet takes about 4,864 GPU hours. Since almost all state-of-the-art methods compared in this work didn't evaluate on ImageNet, the cost of making a thorough comparison on ImageNet is far beyond our budget. 

\subsection{Comparison between RT and SRT}\label{unlabel}

\textbf{Settings.} We compare the performance of RT and SRT on the CIFAR-10 dataset. 
In the pixel-wise adversarial scenario, we use the wide residual network WRN-34-10 \cite{zagoruyko2016}, and set the perturbation size $\epsilon_1= 0.031$. To evaluate the robustness, we apply PGD-20 with step size 0.003 and FGSM under perturbation size 0.031, as suggested in \cite{zhang2019}. 
In the spatial adversarial scenario, we use ResNet \cite{he2016} with 4 residual groups with filter size [16, 16, 32, 64] and 5 residual units each. 
We perform RandAdv and GridAdv attack \cite{engstrom2019} to evaluate the spatial adversarial robustness. The spatial perturbation (\ie, rotation and translation) of RandAdv and GridAdv are determined through random sampling and grid search, respectively. The maximum rotational perturbation is set as $30^\circ$, and we consider translations of at most 3 pixels. More detailed settings (\eg, learning schedule), and additional results on the MNIST dataset will be presented in the \textbf{Appendix (Section 2-3)}.

\textbf{Results.} As shown in Table \ref{SRTbetter(spatial)}-\ref{SRTbetter(pixel)}, SRT is superior to RT under the same conditions, no matter in spatial scenario or pixel-wise scenario. In particular, SRT is much more robust than RT under stronger attack (PGD and GridAdv). For example, the adversarial accuracy of SRT is over more than $5\%$ compared to RT under PGD attack in most cases. This improvement is even more significant under GridAdv, which is at least $17\%$ (more than $20\%$ in most cases). In addition, the clean accuracy of SRT is also higher than that of RT in all cases. As such, we will only use SRT in the following comparisons. 

In particular, an interesting phenomenon is that the clean accuracy does not decrease as the hyper-parameter $\lambda$ increases in the spatial adversarial settings. In other words, there is no trade-off between the spatial adversarial robustness and the standard prediction performance. In contrast, the inherent trade-off between clean accuracy and adversarial accuracy was verified theoretically and empirically in the pixel-wise adversarial settings \cite{tsipras2019}. This interesting contradiction was discussed in \cite{yang2019}, which is out the scope of this paper.

\vspace{-0.3em}
\begin{table*}[ht]
\begin{center}
\small
\caption{Adversarial accuracy ($\%$) of different models under spatial attacks on CIFAR-10 and MNIST dataset. The perturbation is determined through random sampling (RandAdv) or grid search (GridAdv) with rotations and translations considered both together and separately (“RandAdv.R” and “GridAdv.R” for rotations, and “RandAdv.T” and “GridAdv.T” for transformations).}
\vspace{-0.2em}
\label{spaadv}
\scalebox{0.9}{
\begin{tabular}{c|c|c|cc|cc|cc}
\hline
                          & Defense      & Clean & RandAdv & GridAdv & RandAdv.T & GridAdv.T & RandAdv.R & GridAdv.R \\ \hline
\multirow{5}{*}{CIFAR-10} & Standard   & 80.63  & 8.82     & 0.09     & 33.61      & 19.67      & 19.55      & 10.73      \\
                          & AT          & 65.59  & 4.92     & 0.22     & 16.23      & 7.49       & 14.47      & 8.37       \\
                          & Worst-of-$k$ & 82.02  & 70.92    & 54.80    & 75.49      & 69.45      & 73.18      & 68.22      \\
                          & KLR     & 85.40    & 72.77    & 56.28    & 77.43      & 72.71      & 74.80      & 71.04      \\
                          & SRT         & $88.87$  & $\bm{78.99}$      & $\bm{64.83}$    & $\bm{82.16}$       & $\bm{78.47}$       & $\bm{80.84}$       & $\bm{77.24}$       \\ \hline
\multirow{5}{*}{MNIST}    & Standard   & 97.19     & 14.01       & 0.00       & 35.31         & 5.12         & 65.06         & 51.32         \\
                          & AT           & 97.96     & 29.84       & 0.01       & 51.83         & 10.72         & 71.66         & 57.71     \\
                          & Worst-of-$k$ & 98.05     & 94.77       & 84.64       & 96.07         & 93.99         & 95.70 & 94.24         \\
                          & KLR          & 98.43 & 95.26 & 86.08   & 96.63  & 95.07         & 95.92         & 94.48         \\
                          & SRT         & $98.64$  & $\bm{97.02}$      & $\bm{92.12}$    & $\bm{97.68}$       & $\bm{96.85}$       & $\bm{97.33}$       & $\bm{96.54}$       \\ \hline
\end{tabular}
}
\end{center}
\vspace{-1em}
\end{table*}

\begin{table*}[ht]
\begin{center}
\small
\caption{Adversarial accuracy ($\%$) of different models under pixel-wise attacks.}
\vspace{-0.2em}
\label{pixeladv}
\scalebox{0.9}{
\begin{tabular}{c|c|c|ccccccc}
\hline
                          & Defense    & Clean & FGSM           & PGD            & MI-FGSM        & JSMA & C\&W          & Point-wise Attack       & DDNA          \\ \hline
\multirow{6}{*}{CIFAR-10} & Standard & 88.43 & 7.26           & 0              & 0              & 9.98   & 0.14        & 2.06           & 0.01           \\
                          & AT         & 77.17 & 50.35          & 37.37          & 36.97          & 10.80 & 33.42         & {14.27}    & 20.87          \\
                          & TRADES     & 76.23 & 51.98          & 40.20           & 39.79          & 11.90  & 36.09        & 12.32          & 22.37          \\
                          & UAT        & 78.45 & 56.35          & 44.96          & 44.56          & 13.23  & 40.03         & 12.62          & 24.08          \\
                          & RST        & 79.99 & \textbf{59.55} & {48.38}    & {47.97}    & {13.78}  & 42.99   & 13.42          & \textbf{27.73} \\
                          & SRT       & 78.46 & {59.34}    & \textbf{48.66} & \textbf{48.24} & \textbf{16.99}  & \textbf{43.33} & \textbf{18.80} & {27.71}    \\ \hline
\multirow{6}{*}{MNIST}    & Standard & 99.01 & 41.06          & 2.87           & 4.37           & 10.18  & 0.01        & 0.04           & 16.44          \\
                          & AT         & 98.99 & 96.61          & 94.69          & 93.57          & 20.99  & 94.67        & 2.47           & 95.35          \\
                          & TRADES     & 98.99 & 96.92          & 95.12          & 93.98          & 18.48  & 92.37        & 2.08           & 94.12          \\
                          & UAT        & 99.16 & {97.51}    & {96.14}    & {95.65}    & 23.79  & 96.16        & {5.52}     & \textbf{96.39} \\
                          & RST        & 98.83 & 97.21          & 95.47          & 95.05          & {26.63} & 95.34   & 3.30           & 94.25          \\
                          & SRT       & 99.28 & \textbf{97.77} & \textbf{96.60} & \textbf{95.79} & \textbf{26.79} & \textbf{96.19} & \textbf{5.81}  & {96.10}   \\ \hline
\end{tabular}
}
\end{center}
\vspace{-1.2em}
\end{table*}

\subsection{Spatial Adversarial Defense}\label{spatialdefense}
\textbf{Settings.} We select Worst-of-$k$ \cite{engstrom2019} and KL divergence based regularization (KLR) \cite{yang2019} as the baseline methods in the following evaluations. 
Those methods achieve the state-of-the-art results in the realm of spatial adversarial defense. Besides, we also provide the standard training model (Standard) and the pixel-wise adversarial training (AT) \cite{madry2017} as other important baselines for reference. 
The settings of hyper-parameters of baseline methods follow the suggestions in their original manuscripts.
For the proposed SRT method, we set $\lambda =0.2$ on both CIFAR-10 and MNIST. Other detailed settings, $e.g.$, learning rate and iterations, will be presented in the \textbf{Appendix (Section 4)}.

\textbf{Results.} As shown in Table \ref{spaadv}, SRT significantly exceeds all existing methods. For example, SRT achieves more than $5\%$ improvement under any attacks on the CIFAR-10 dataset, compared with the current state-of-the-art method, the KLR. In addition, similar to the phenomenon in Section \ref{unlabel}, the clean accuracy of SRT is also higher compared to all existing methods. This improvement of clean accuracy is presumably because the spatial adversarial defense serve as an effective data augmentation approach. Note that the pixel-wise adversarial defense (\ie, AT) has almost no benefit to the spatial adversarial robustness. In other words, pixel-wise adversarial robustness does not necessarily mean general adversarial robustness. Therefore, developing a general adversarial framework toward general adversarial robustness is of great significance.

\vspace{-0.5em}
\subsection{Pixel-wise Adversarial
Defense}\label{pixeldefense}
\vspace{-0.5em}

\textbf{Settings.} We compare SRT with PGD-based AT \cite{madry2017}, trade-off inspired adversarial defense (TRADES) \cite{zhang2019}, unsupervised adversarial training (UAT) \cite{tsipras2019} and robust self-training (RST) \cite{carmon2019}. TRADES is the state-of-the-art supervised defense, while UAT and RST are the most advanced semi-supervised adversarial defenses. In SRT, we set $\lambda = 1$ on both CIFAR-10 and MNIST dataset. To evaluate the robustness, we apply FGSM \cite{goodfellow2014}, PGD \cite{kurakin2016}, momentum iterative attack (MI-FGSM) \cite{dong2018boosting} under $\ell_\infty$ norm. We also apply some other attack methods, including jacobian saliency map attack (JSMA) \cite{papernot2016limitations}, Carlini \& Wagner attack (C\&W) \cite{carlini2017}, point-wise attack \cite{schott2018}, and decoupled direction and norm attack (DDNA) \cite{rony2019} for the evaluation. Other settings are shown in the \textbf{Appendix (Section 5)}.

\textbf{Results.}  As shown in Table \ref{pixeladv}, compared to the supervised adversarial defense methods, semi-supervised defenses (\ie, UAT, RST, and SRT) achieve significant improvement in both clean accuracy and adversarial robustness with additional unlabeled samples. Among three different semi-supervised defenses, SRT reaches the best performance. Specifically, on the CIFAR-10 and MNIST datasets, SRT has the best adversarial robustness in most cases. In the case where a few are not the best, SRT is still the sub-optimal. In addition, although UAT and RST perform similarly to and even exceed SRT in a few cases, their superiority is not consistent across different datasets. For example, RST has better performance on CIFAR-10, whereas UAT is better on MNIST.

\comment{
\begin{figure}[t]
\small
\centering
\vskip -0.15in
\includegraphics[width=0.47\textwidth]{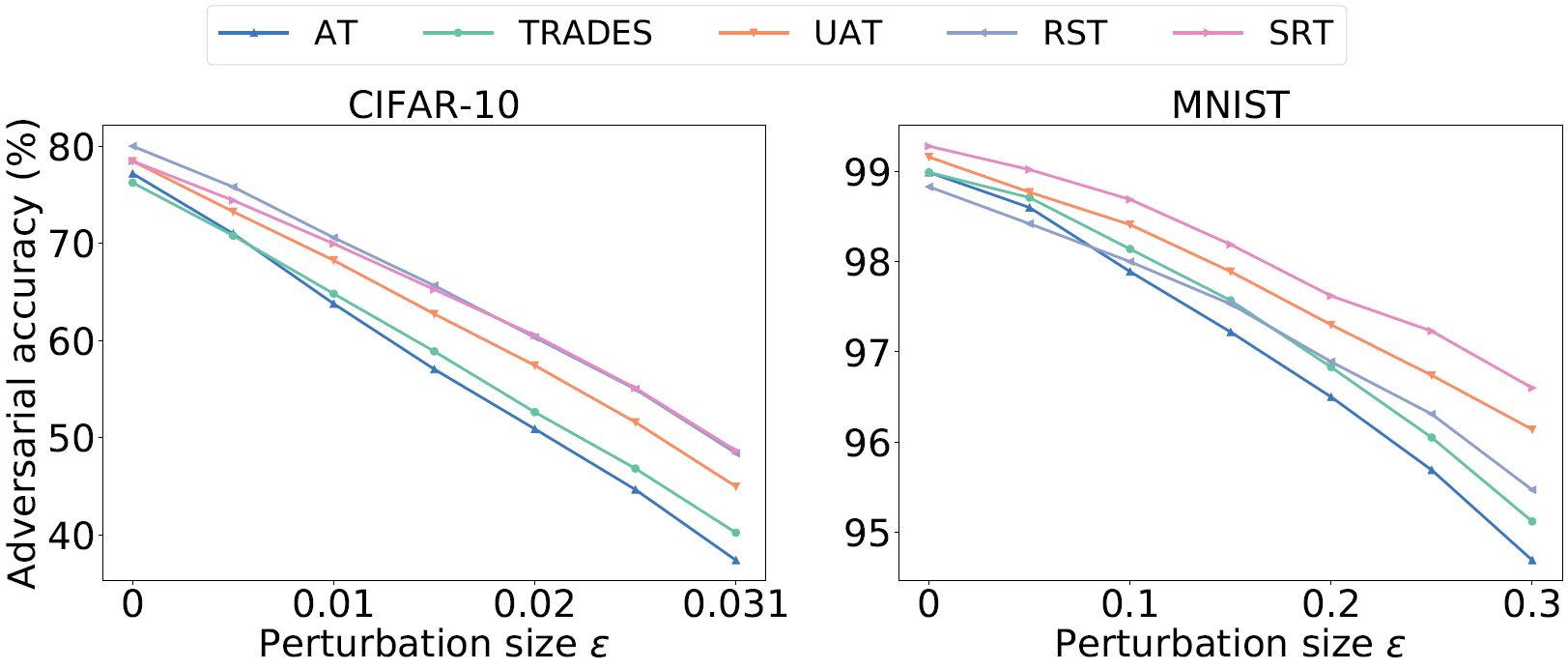}
\caption{Comparison between pixel-wise adversarial defense methods under PGD with different perturbation size. Perturbation size $\epsilon$ denotes the upper bound of the perturbation (see Eq.(\ref{label: definition of neighborhood})).}
\label{fig4} 
\vskip -0.15in
\end{figure}

We also compare different defense methods under PGD with different perturbation size. As shown in Figure \ref{fig4}, the curve of the SRT is almost always at the top. In other words, SRT has the best adversarial robustness under different penetration sizes in most cases. Especially on the MNIST dataset, SRT significantly outperforms all other methods. In particular, although the robustness of RST is not far from SRT on CIFAR-10, and its clean accuracy is sometimes even higher than SRT, it requires much more computational resources since self-training requires training an additional intermediate model.
}

\begin{figure*}[ht]
\small
\centering
\vspace{-1em}
\includegraphics[width=0.85\textwidth]{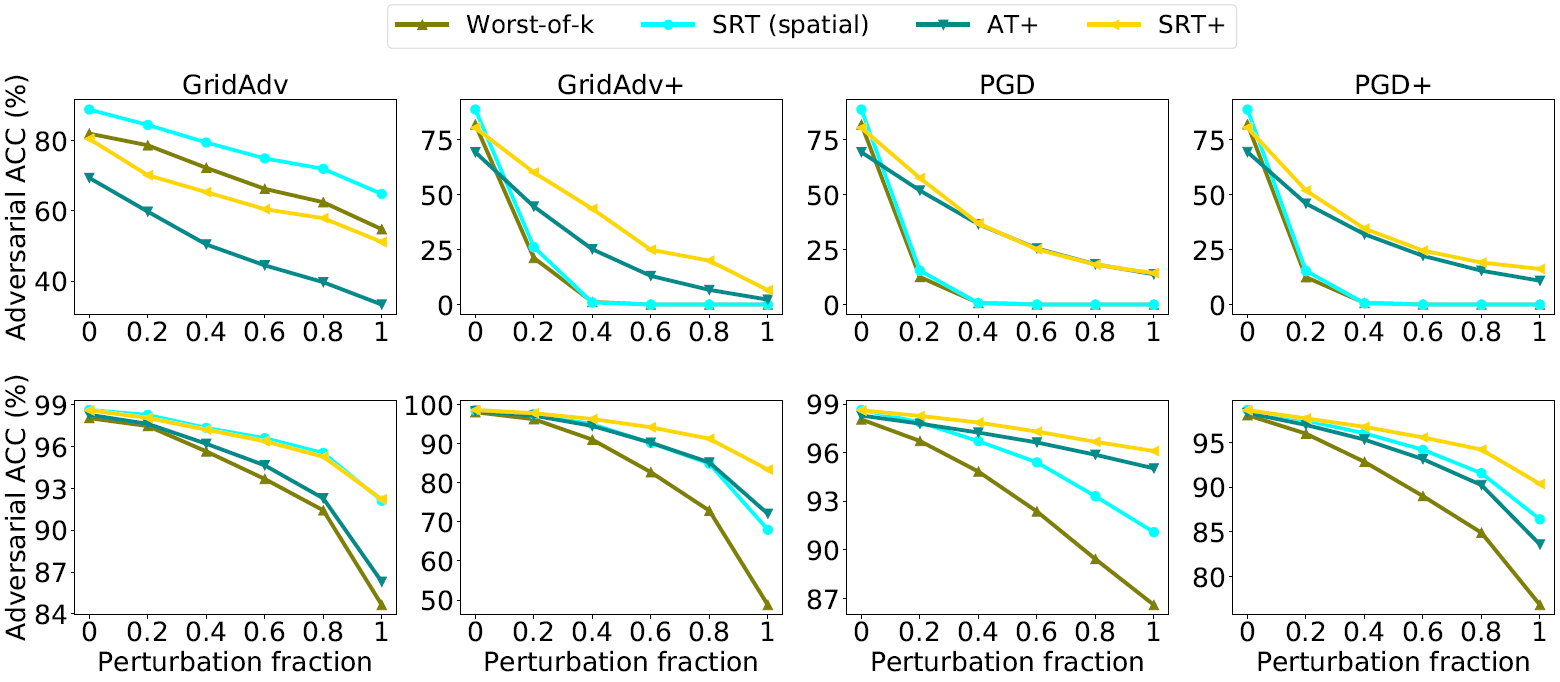}
\caption{Compound adversarial attacks and defenses. \textbf{First row}: results on the CIFAR-10 dataset. \textbf{Second row}: results on the MNIST dataset. The maximum perturbation size is the same as the one used in previous experiments. Perturbation fraction represents the ratio of the current perturbation size to the previous maximum perturbation size.}
\label{fig5}
\vspace{-1em}
\end{figure*}

\subsection{Compound Adversarial Attack and Defense}
\textbf{Settings.} Here we evaluate the adversarial robustness to the attack consisting of multiple types of adversarial attacks, dubbed compound attack\footnote{We notice that the name ``compound attack'' was also used in the area of network security, indicating the combined attack of multiple cyber attacks. \cite{guruswamy2012compound}}. 
We adopt two types of compound attacks, including \textbf{(1)} combining the spatial attack GridAdv \cite{engstrom2019} on the top of the pixel-wise attack PGD \cite{kurakin2016} (dubbed PGD+), and \textbf{(2)} combining PGD on the top of GridAdv (dubbed GridAdv+). 
Accordingly, we also provide two compound defenses, including \textbf{(1)} combining the spatial defense Worst-of-$k$ \cite{engstrom2019} on the top of the pixel-wise defense AT \cite{madry2017} (AT+), and \textbf{(2)} combining the spatial SRT on the top of the pixel-wise SRT (SRT+). Other detailed settings will be presented in the \textbf{Appendix (Section 6)}.

\textbf{Results.} As shown in Figure \ref{fig5}, compared to the single type of attack, compound adversarial attacks (GridAdv+ and PGD+) have a much stronger threat under the same conditions. This phenomenon indicates that by simply combining different types of attacks, a powerful attack can be constructed, which poses a huge threat to DNNs. In particular, the spatial adversarial defenses have limited benefit to the pixel-wise defense, which can be observed by the subgraphs in third column. Especially on the CIFAR-10 dataset, the adversarial accuracy could quickly drop to zero under PGD attack with small perturbation fraction. This again confirms that a single type of defense may not have much effect on defending against another type of attack. Besides, the results also suggest that the model trained with the general perturbed neighborhood (\ie, AT+ and SRT+) could simultaneously defend different types of perturbations, \ie, the compound attack. 
Compared to single-type defense methods, compound defenses achieve significant improvements defending against compound attacks, as well good defending performance to single attacks.
In addition, in the case of compound defense, our method (SRT+) is also better than the extension of previous methods (AT+).

\subsection{The Effect of Unlabeled Data}\label{unlabelimpact}

In this section, we use PGD and GridAdv to evaluate the adversarial robustness under pixel-wise and spatial scenario respectively. Except for the number of unlabeled samples, all settings are the same as those used in Section \ref{spatialdefense}-\ref{pixeldefense}.

\begin{figure*}[ht]
\small
\vspace{-1em}
\centering
\subfigure[CIFAR-10]{
\includegraphics[width=0.49\textwidth]{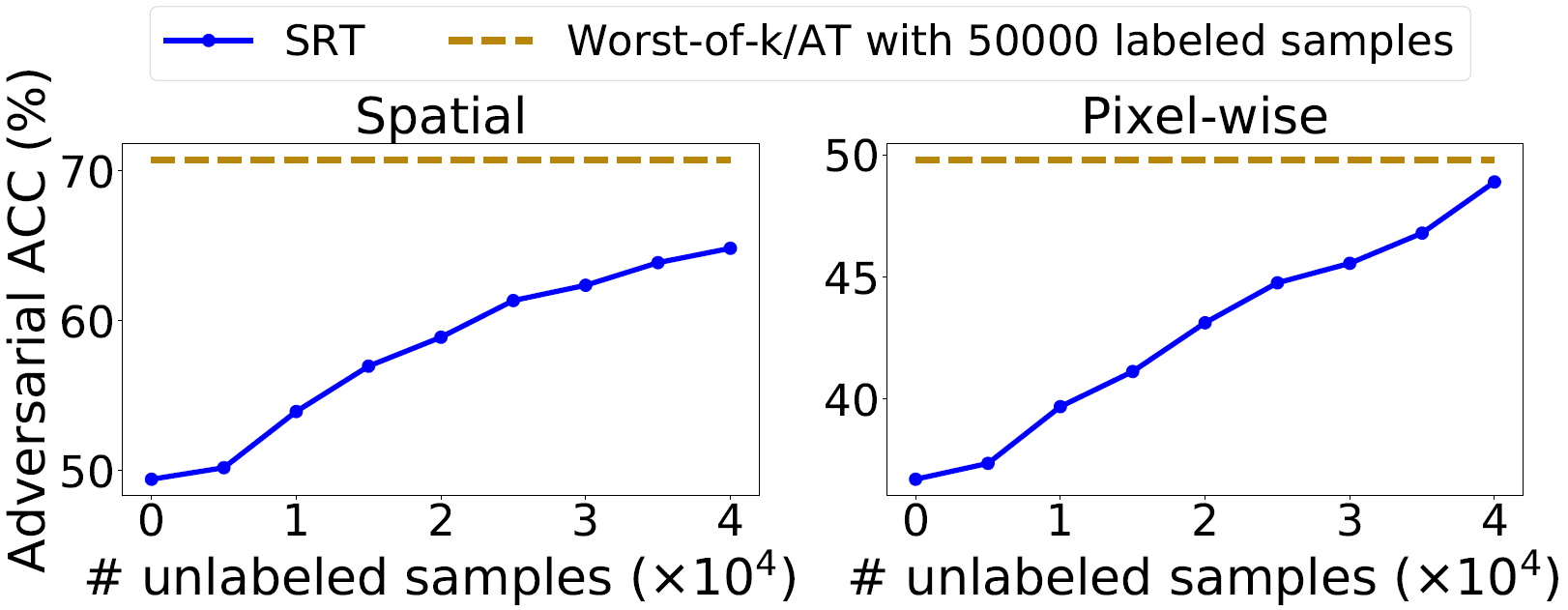}}
\subfigure[MNIST]{
\includegraphics[width=0.49\textwidth]{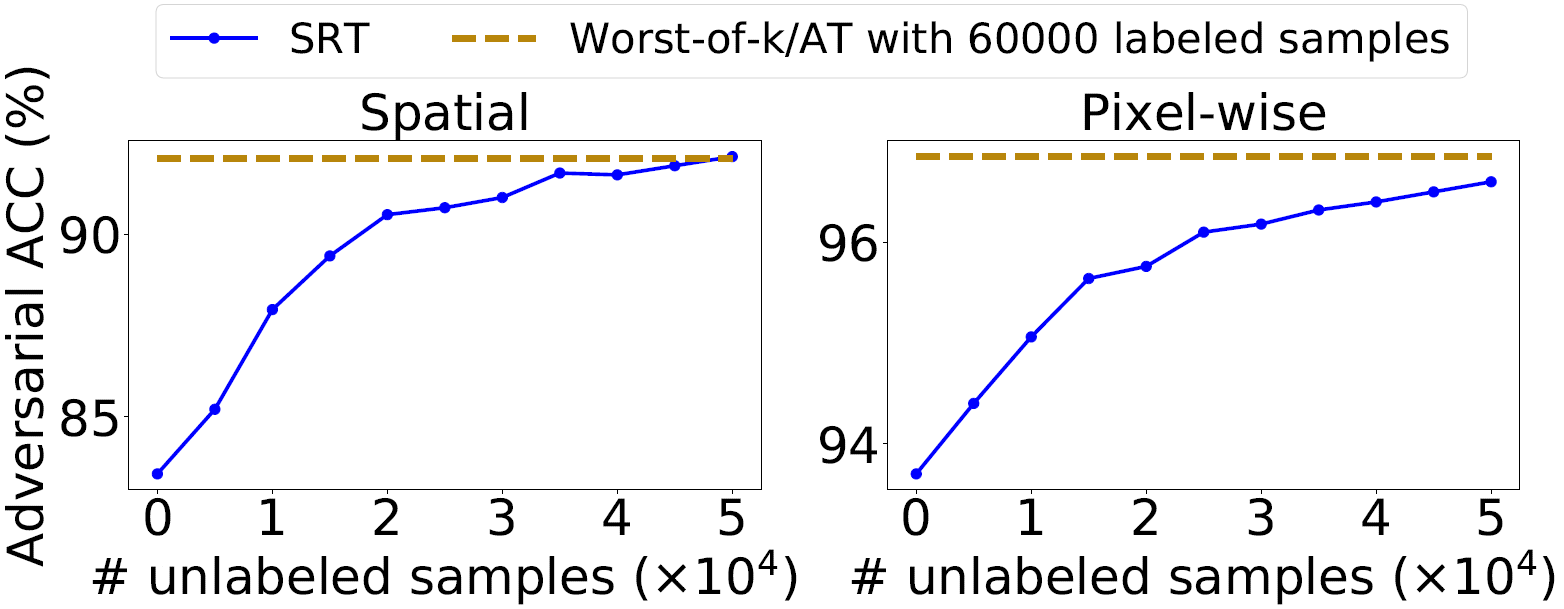}}
\vspace{-0.5em}
\caption{Adversarial accuracy w.r.t the number of unlabeled data.}
\vspace{-1em}
\label{fig2} 
\end{figure*}

\begin{figure*}[ht]
\small
\vspace{-1.2em}
\centering
\subfigure[Spatial adversarial scenario]{
\includegraphics[width=0.48\textwidth]{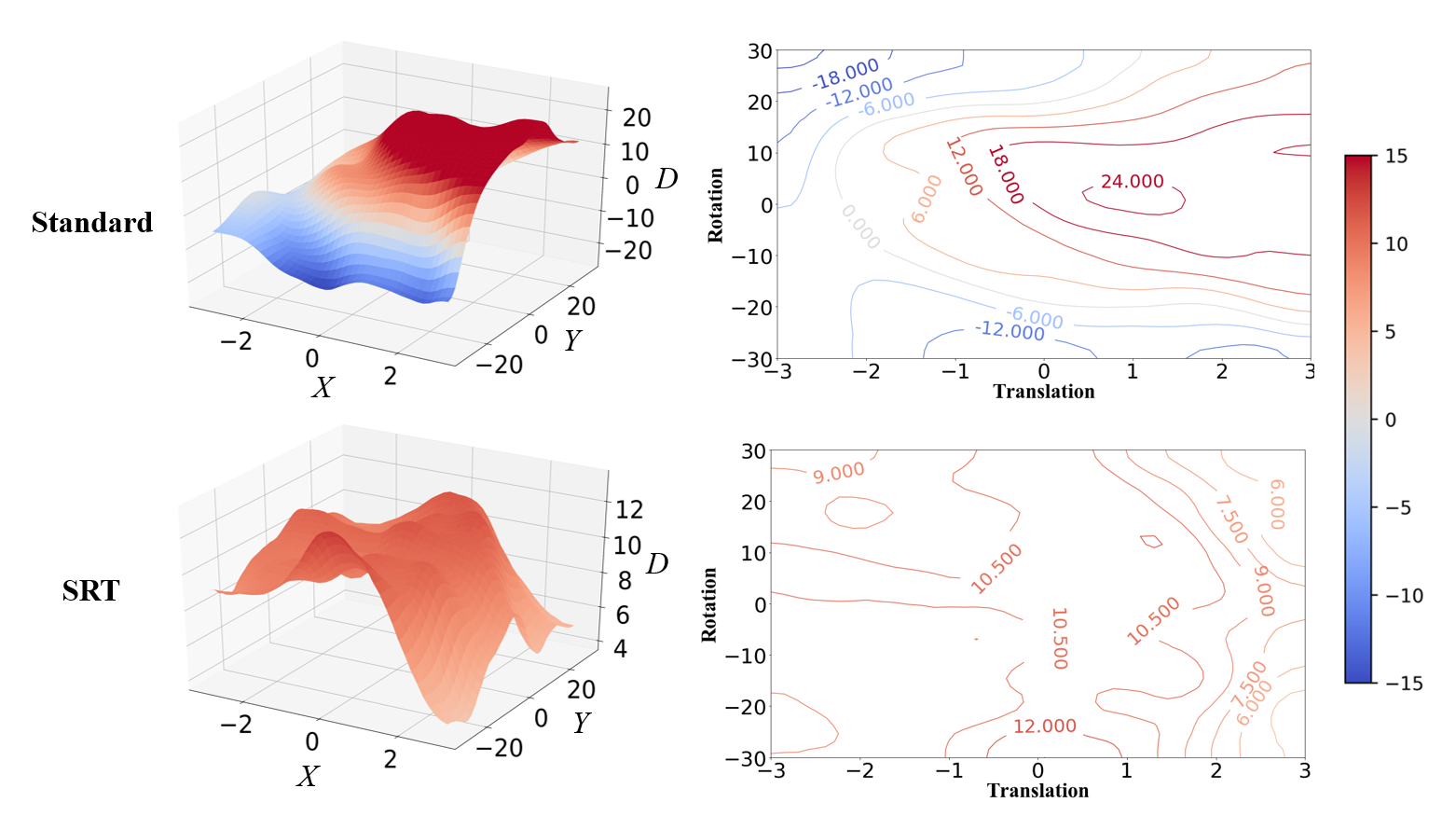}}
\subfigure[Pixel-wise adversarial scenario]{
\includegraphics[width=0.48\textwidth]{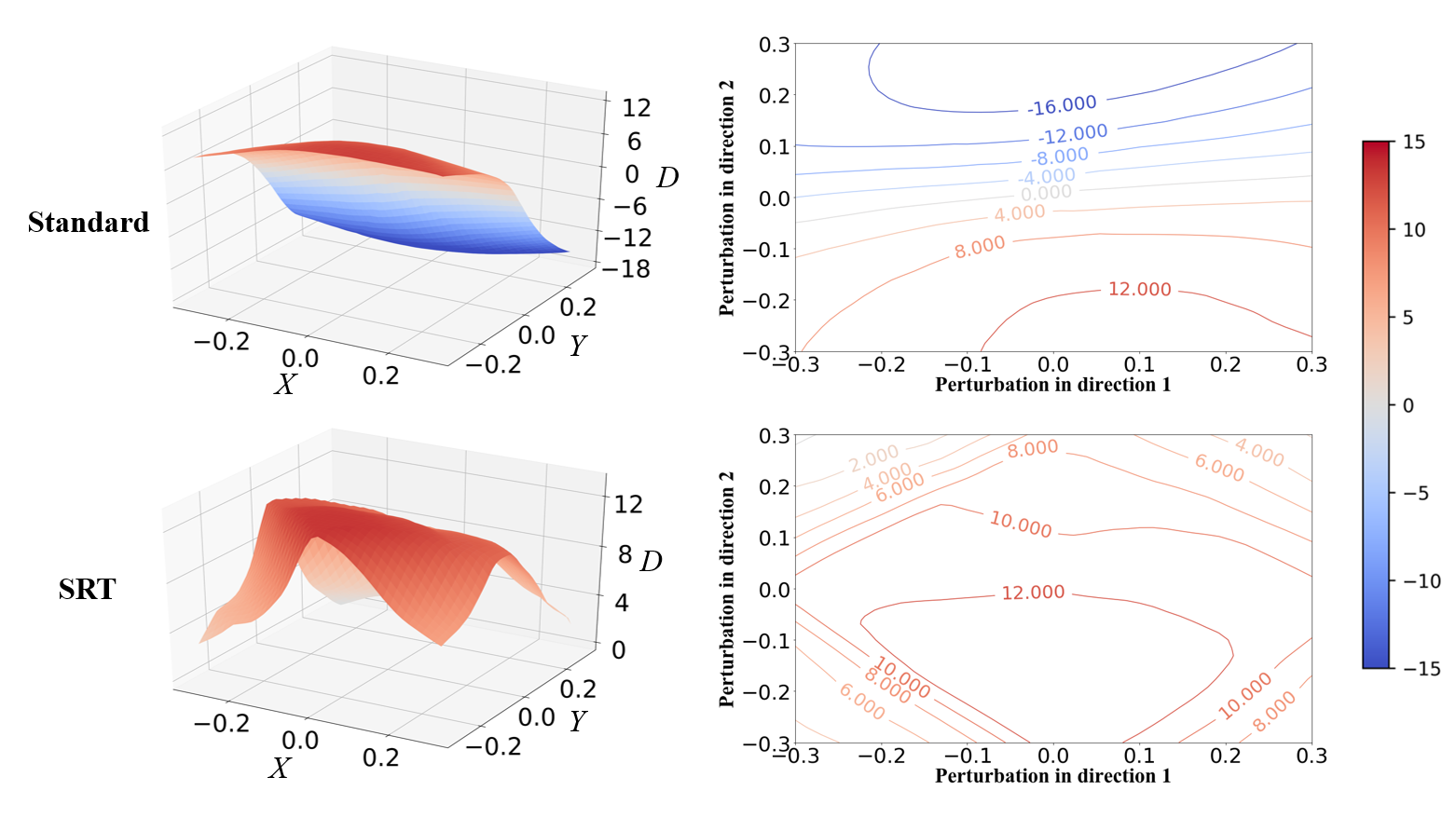}}
\vspace{-0.5em}
\caption{Comparison between the decision surfaces of the SRT and those of the standard training model on the MNIST dataset. \textbf{First row}: the 3D decision surface and their corresponding 2D version of the standard training model. \textbf{Second row}: the decision surface of the model trained with SRT. In the 3D decision surfaces, the $X$ and $Y$ axis represent two different perturbation directions and its value indicates the perturbation size, while the $Z$-axis indicates the decision value. If and only if the decision value is positive, the prediction is correct. In those figures, the red area indicates the correctly classified region, while the blue area is the misclassified region.} 
\label{fig3}
\vspace{-1.5em}
\end{figure*}

As shown in Figure \ref{fig2}, the adversarial accuracy of models trained with SRT increases with the number of unlabeled samples in all settings. Moreover, leveraging large amounts of unlabeled examples, SRT achieves similar adversarial robustness to the supervised adversarial defense methods (AT and Worst-of-$k$) trained with all labeled samples in the original training set (\ie, 60,000 samples on MNIST, and 50,000 samples on CIFAR-10). 
In particular, in the setting of spatial attacks on MNIST, the adversarial accuracy of SRT-based model trained with 10,000 labeled samples and 50,000 unlabeled samples even exceeds that of Worst-of-$k$ with 60,000 labeled samples. Besides, the adversarial accuracy still has an upward trend at the end of the curve (\ie, using all unlabeled data), which implies that the robustness can be further improved if additional samples are utilized.

\vspace{-0.3em}
\subsection{Decision Surface}
\vspace{-0.3em}

In this section, we verify the effectiveness of SRT through visualizing of the geometry of the decision surface on MNIST dataset. The decision value $D$ of a sample $\bm{x}$ is defined as $D(\bm{x})=p_y - \max_{i \neq y} p_i$, where $f_{\bm{w}}(\bm{x}) = [p_1; \cdots; p_K]$, $D(\bm{x})>0$ indicates that the prediction of $\bm{x}$ is correct and vice versa. We visualize the decision surface instead of the loss surface, since it can be further used to evaluate the prediction performance. The decision surface on CIFAR-10, and the visualization of loss surfaces is shown in the \textbf{Appendix (Section 7)}.

As shown in Figure \ref{fig3}, the geometric property of decision surfaces of SRT-based model and that of the standard training model are significantly different. Compared to SRT, the decision surfaces of the standard training model have sharper peaks and larger slopes, which explains that the prediction of this model is vulnerable to small perturbation. In contrast, the surfaces of SRT are relatively flat and located on a plateau with positive decision confidence around the benign sample, though the decision surface in the spatial scenario is more rugged than that in the pixel-wise scenario. Consequently, the output of SRT still lies in the region of correct classification if perturbed with a certain range, which explains the adversarial robustness of SRT.

\vspace{-0.5em}
\section{Conclusion}
\vspace{-0.5em}
In this paper, we propose a novel defense method, dubbed robust training (RT), by jointly minimizing the standard risk and robust risk. We prove that $R_{adv}$ is upper-bounded by $R_{stand} + R_{rob}$, which implies that RT has the similar effect as adversarial training. In addition, we extend RT to the semi-supervised mode (\ie, SRT) to further enhance the adversarial robustness, due to the fact that the robust risk is independent of the ground-truth label. Moreover, we extend the $\ell_{p}$-bounded neighborhood to a general case, which covers different types of perturbations. Consequently, the model trained using SRT with general perturbed neighborhood could simultaneously defend different types of perturbations. Extensive experiments verify the superiority of SRT for defensing pixel-wise or spatial perturbations separately, as well as both perturbations simultaneously. Note that our extension of the perturbed neighborhood is general and not limited to the pixel-wise and spatial scenario. More types of transformations, such as blurring and color shifting, could be also covered. It will be explored in our future work.

\newpage
\section{Broader Impact}

Adversarial defenses focus on the security issues of machine learning (ML), which is the cornerstone before ML algorithms can be safely adopted. As such, as one of the adversarial defense, our work has positive impacts in general.

Specifically, from the aspect of positive broader impacts, \textbf{(1)} our work can be adopted in the classification systems to alleviate the risk of being successfully attacked; \textbf{(2)} our work theoretically analyzed the relationship among standard risk, robust risk, and adversarial risk, which builds a bridge between the robust optimization community and the adversarial learning community. As such, it may probably promote further cooperations in these two areas; \textbf{(3)} we also verified the potential of unlabeled data in the adversarial defense and proposed a method to do so, which may inspire deeper considerations about how to generalize our approach in other applications ($e.g.$, video classification); \textbf{(4)} we also demonstrate the serious threat of compound adversarial attack and propose a method to alleviate it, which may inspire more comprehensive research on this area; \textbf{(5)} Since the adversarial vulnerability comes from the difference between human visual system and algorithms, our work may provide a new angle toward understanding such difference. 

From the aspect of negative broader impacts, \textbf{(1)} attackers may adopt the proposed compound approach to create more malicious attacks; \textbf{(2)} the proposed method further promotes the requirement of (unlabeled) data, which may further enhance the concerns about the data privacy.

\bibliographystyle{unsrt}
\bibliography{ref}

\newpage

\begin{center}
    \begin{Large}
        \textbf{Appendix}
    \end{Large}
\end{center}

\begin{appendices}

\setcounter{lemma}{0}
\setcounter{theorem}{0}

\section{Proofs}

\begin{lemma}\label{threerisk}
Standard Risk, adversarial risk and robust risk of a sample $\bm{x}$ are correlated. Specifically,
\begin{equation}
    R_{adv}(\bm{x}) = R_{stand}(\bm{x}) + \left(1 - R_{stand}(\bm{x})\right)R_{rob}(\bm{x}).  
\end{equation}
\end{lemma}

\begin{proof}

According to the definition of $R_{adv}$ (see Definition 2 in the main manuscript), we have 
\begin{equation}\label{step1}
    R_{adv}(\bm{x})  = \max \limits_{\bm{x'} \in \mathcal{N}_{\delta, \bm{T}}(\bm{x})} \mathbb{I}\{C(\bm{x'}) \neq y\}.
\end{equation}
Since $\max \limits_{\bm{x'} \in \mathcal{N}_{\delta, \bm{T}}(\bm{x})} \mathbb{I}\{C(\bm{x'}) \neq y\} \in \{0, 1\}$, and its value is 1 \emph{iff} $ \exists \bm{x}' \in \mathcal{N}_{\delta, \bm{T}}(\bm{x})\, s.t. \, C(\bm{x}') \neq y$. Therefore, we have
\begin{equation}\label{step2}
    (\ref{step1}) = \mathbb{I}\left\{ \exists \bm{x}' \in \mathcal{N}_{\delta, \bm{T}}(\bm{x})\, s.t. \, C(\bm{x}') \neq y \right\}.
\end{equation}

Event $\bm{A}$: $\exists \bm{x}' \in \mathcal{N}_{\delta, \bm{T}}(\bm{x})\, s.t. \, C(\bm{x}') \neq y$ can be divided into two disjoint sub-events, as follows:
$$
\begin{array}{c}
     \bm{A_1}: (\exists \bm{x}' \in \mathcal{N}_{\delta, \bm{T}}(\bm{x})\, s.t. \, C(\bm{x}') \neq y) \cap (C(\bm{x}) \neq y),  \\
     \bm{A_2}: (\exists \bm{x}' \in \mathcal{N}_{\delta, \bm{T}}(\bm{x})\, s.t. \, C(\bm{x}') \neq y) \cap (C(\bm{x}) = y). 
\end{array}
$$
Let $\bm{B}: \exists \bm{x}' \in \mathcal{N}_{\delta, \bm{T}}(\bm{x})\, s.t. \, C(\bm{x}') \neq y$, $\bm{C}: C(\bm{x}) = y$, now we prove that $\bm{A_1}$ is true if and only if event $\overline{\bm{C}}: C(\bm{x}) \neq y$ is true. 

\noindent {\bf 1)} Suppose $\overline{\bm{C}}$ is true, since $\bm{x} \in \mathcal{N}_{\delta, \bm{T}}(\bm{x})$, we have
$$
\exists \bm{x}' \in \mathcal{N}_{\delta, \bm{T}}(\bm{x})\, s.t. \, C(\bm{x}') \neq y
$$
holds, $i.e.$, $\bm{B}$ is true, therefore $\bm{A_1} = \bm{B} \cap \overline{\bm{C}}$ is true.

{\bf 2)} Suppose $\overline{\bm{C}}$ is false, then $\bm{A_1} = \bm{B} \cap \overline{\bm{C}}$ is false.

{\bf 3)} Suppose $\bm{A_1} = \bm{B} \cap \overline{\bm{C}}$ is true, then $\overline{\bm{C}}$ is true.

{\bf 4)} Suppose $\bm{A_1} = \bm{B} \cap \overline{\bm{C}}$ is false, we can prove that $\overline{\bm{C}}$ is false by seeking the contradiction.

\noindent Thus, we have
\begin{equation}\label{case1}
    \mathbb{I}\{\bm{A_1}\}=\mathbb{I}\{C(\bm{x}) \neq y\}.
\end{equation}
Since $\bm{A_2} = \bm{B} \cap \bm{C} = (\bm{B} \cap \bm{C}) \cap \bm{C}$ and
$$
    (\bm{B} \cap \bm{C}): \exists \bm{x}' \in \mathcal{N}_{\delta, \bm{T}}(\bm{x})\, s.t. \, C(\bm{x}') \neq C(\bm{x}),
$$
\begin{equation}
\begin{array}{cl}\label{case2}
    \mathbb{I}\{\bm{A_2}\} & = \mathbb{I}\{\bm{C}\} \times \mathbb{I}\{\bm{B}\cap \bm{C}\} \\
     & = \mathbb{I}\{C(\bm{x}) = y\} \times \mathbb{I}\left\{ \exists \bm{x}' \in \mathcal{N}_{\delta, \bm{T}}(\bm{x})\, s.t. \, C(\bm{x}') \neq C(\bm{x}) \right\}.
\end{array}
\end{equation}
Therefore, combining Eqs. (\ref{case1}) and (\ref{case2}), we obtain
\begin{equation}\label{step4}
\begin{array}{rl}
    (\ref{step2}) & =  \mathbb{I}\{\bm{A_1}\} + \mathbb{I}\{\bm{A_2}\}\\
    & = \mathbb{I}\{C(\bm{x}) \neq y\} + \mathbb{I}\{C(\bm{x})=y\} \times \\
     & \quad \mathbb{I}\left\{ \exists \bm{x}' \in \mathcal{N}_{\delta, \bm{T}}(\bm{x})\, s.t. \, C(\bm{x}') \neq C(\bm{x}) \right\}.
\end{array}
\end{equation}
Similar to Eq. (\ref{step2}), we have 
\begin{equation} \label{step7}
   \mathbb{I}\left\{ \exists \bm{x}' \in \mathcal{N}_{\delta, \bm{T}}(\bm{x})\, s.t. \, C(\bm{x}') \neq C(\bm{x}) \right\} = 
   \max \limits_{\bm{x'} \in \mathcal{N}_{\delta, \bm{T}}(\bm{x})} \mathbb{I}\{C(\bm{x'}) \neq C(\bm{x})\}.
\end{equation}
Combining Eqs. (\ref{step4}) and (\ref{step7}), we have
\begin{equation}\label{step5}
    (\ref{step4})  = \mathbb{I}\{C(\bm{x}) \neq y\} + \mathbb{I}\{C(\bm{x})=y\} \times \max \limits_{\bm{x'} \in \mathcal{N}_{\delta, \bm{T}}(\bm{x})} \mathbb{I}\{C(\bm{x'}) \neq C(\bm{x})\}.
\end{equation}
According to the definitions of $R_{adv}$, $R_{rob}$ and $R_{stand}$,
\begin{equation}
    (\ref{step5}) = R_{stand}(\bm{x}) + \left(1 - R_{stand}(\bm{x})\right)R_{rob}(\bm{x}).
\end{equation}
\end{proof}

\begin{table}[!ht]
\vspace{-0.3em}
\small
  \begin{minipage}{0.49\linewidth}
    \centering
    \caption{Comparison between SRT and RT under spatial attacks on MNIST dataset.}
    \begin{tabular}{l|c|c|c}
    \hline
        & Clean & RandAdv & GridAdv  \\ \hline
    Standard & 97.19   &  71.00  &  40.49 \\ \hline
    RT ($\lambda = 0.15$) & 98.47    & 92.88   & 72.85  \\
    SRT ($\lambda = 0.15$) & $\bm{98.61}$    & $\bm{96.85}$   & $\bm{91.52}$  \\ \hline
    RT ($\lambda = 0.20$) & 98.33    & 93.66   & 76.68  \\
    SRT ($\lambda = 0.20$) & $\bm{98.64}$    & $\bm{97.02}$   & $\bm{92.12}$ \\ \hline
    RT ($\lambda = 0.25$) &  98.29    & 93.91  & 78.00 \\
    SRT ($\lambda = 0.25$) & $\bm{98.63}$     & $\bm{96.91}$    & $\bm{91.70}$ \\ \hline
    RT ($\lambda = 0.30$) & 98.42    & 93.86   & 77.74 \\
    SRT ($\lambda = 0.30$) & $\bm{98.62}$     & $\bm{97.08}$    & $\bm{91.44}$  \\ \hline
    \end{tabular}
    \label{ap_SRTbetter(spatial)}
  \end{minipage}
  \quad
  \begin{minipage}{0.47\linewidth}
    \centering
    \caption{Comparison between SRT and RT under pixel-wise attacks on MNIST dataset.}
    \begin{tabular}{l|c|c|c}
    \hline
        & Clean & FGSM & PGD  \\ \hline
    Standard & 99.02    & 93.80   & 86.12  \\ \hline
    RT ($\lambda = 0.2$) & 99.06   & 97.18   & 95.84 \\
    SRT ($\lambda = 0.2$) & $\bm{99.31}$    & $\bm{98.10}$   & $\bm{97.18}$ \\ \hline
    RT ($\lambda = 0.4$) & 99.14    & 97.57   & 96.23 \\
    SRT ($\lambda = 0.4$) & $\bm{99.40}$    & $\bm{98.28}$   & $\bm{97.55}$ \\ \hline
    RT ($\lambda = 0.6$) & 99.06    & 97.78   & 96.92 \\
    SRT ($\lambda = 0.6$) & $\bm{99.34}$    & $\bm{98.47}$   & $\bm{97.81}$ \\ \hline
    RT ($\lambda = 0.8$) & 99.11    & 97.90   & 97.06  \\
    SRT ($\lambda = 0.8$) & $\bm{99.35}$     & $\bm{98.53}$    & $\bm{97.86}$ \\ \hline
    \end{tabular}
    \label{ap_SRTbetter(pixel)}
  \end{minipage}
\vspace{-0.3em}
\end{table}

\begin{theorem}
The relationship between the adversarial training and the robust training is as follows
\begin{equation}
    \min_{\boldsymbol{w}} R_{adv}(\mathcal{D}) \leq \min_{\boldsymbol{w}} \{R_{stand}(\mathcal{D}) + R_{rob}(\mathcal{D})\}.
\end{equation}
\end{theorem}
\begin{proof}
Since $R_{stand}(\bm{x}), R_{rob}(\bm{x}) \in [0,1]$, 
according to Lemma 1, it is easy to obtain that
\begin{equation}\label{sampleineq}
    R_{adv}(\bm{x})  \leq R_{stand}(\bm{x}) + R_{rob}(\bm{x}). 
\end{equation}
By seeking the expectation on both sides of (\ref{sampleineq}), we obtain 
\begin{equation}\label{rel}
    R_{adv}(\mathcal{D})  \leq R_{stand}(\mathcal{D}) + R_{rob}(\mathcal{D}). 
\end{equation}
Utilizing the fact that all three risks are non-negative, we further obtain 
\begin{equation}
\min_{\boldsymbol{w}} R_{adv}(\mathcal{D}) \leq \min_{\boldsymbol{w}} \{R_{stand}(\mathcal{D}) + R_{rob}(\mathcal{D})\}.
\end{equation}
\end{proof}

\section{Settings for Comparison between RT and SRT}\label{set:compare}
\noindent \textbf {Training Setup.} In the spatial adversarial settings, we use a standard ResNet \cite{he2016} with 4 residual groups with filter size [16, 16, 32, 64] and 5 residual units each. We set the learning rate $\eta_2 = 0.1$, batch size $m_2 = 128$, and run 80,000 iterations. In the pixel-wise experiments, we use the wide residual network WRN-34-10 \cite{zagoruyko2016}, and set the perturbation size $ \epsilon_1= 0.031$, perturbation step size $\alpha_1 = 0.007$, number of iterations $K_1 = 10$ (for the inner maximization sub-problem), learning rate $\eta_1 = 0.1$, batch size $m_1 = 128$, and run 100 epochs in the training dataset. 

\noindent \textbf{Data Preprocessing.} We conduct standard data augmentation techniques for benign images. Specifically, 4-pixel padding is used before performing random crops of size $32\times32$, followed by a left-and-right random flipping with probability 0.5. 

\noindent \textbf {Attack Setup.} We perform RandAdv and GridAdv attack \cite{engstrom2019} to evaluate the spatial adversarial robustness. The spatial perturbation (\ie, rotation and translation) of RandAdv and GridAdv are determined through random sampling and grid search, respectively. Specifically, the maximum rotational perturbation is set as $30^\circ$, and we consider translations of at most 3 pixels. For the GridAdv, we consider 5 and 31 values equally spaced per direction for translation and rotation, respectively. For the RandAdv, we perform a uniformly random perturbation from the attack spaced used in GridAdv. These settings are based on those suggested in \cite{engstrom2019}. To evaluate the pixel-wise adversarial robustness, we apply PGD-20 with step size 0.003 and FGSM under perturbation size 0.031, as suggested in \cite{zhang2019}.

\section{Comparison between SRT and RT on MNIST Dataset}\label{compare:MNIST}
\noindent \textbf{Setup}. We use a simple CNN with four convolutional layers followed by three fully-connected layers \cite{zhang2019} in both pixel-wise and spatial scenario. In spatial scenario, except for the model architecture, the other settings are the same as those stated in Section \ref{set:compare}. In the pixel-wise experiments, we set perturbation size as 0.1, perturbation step size as 0.01, number of iterations as 20, with learning rate $\eta_3 = 0.01$, batch size $m_3 = 128$, and run 50 epochs in the training dataset. We apply PGD-40 with step size 0.005 and FGSM under perturbation size 0.1, which are suggested in \cite{zhang2019}. 

\noindent \textbf {Results.} Similar to the results on the CIFAR-10 dataset, SRT is superior to RT under the same conditions on the MNIST dataset. SRT is also much more robust than RT under stronger attack (PGD and GridAdv), especially in the spatial scenario.

\section{Settings for Spatial Adversarial Defense}
\noindent \textbf{Setup}. 
The training and the attack setup follow the same settings used in Section \ref{set:compare}-\ref{compare:MNIST}. In SRT, the we set $\lambda =0.2$ on both CIFAR-10 and MNIST according to previous experiments, and the hyper-parameters of baseline methods are set according to their paper.

\section{Settings for Pixel-wise Adversarial Defense}
\noindent \textbf{Training Setup}.
The model architectures are the same as those used in Section \ref{set:compare}-\ref{compare:MNIST}. To ensure that all methods have converged, we train 120 epochs in this experiment, and the initial learning rate is changed to 0.05. The number of iterations and the perturbation size is modified to 40 and 0.3 on MNIST respectively, as suggested in \cite{zhang2019}. Other settings are the same as those used in Section \ref{set:compare}-\ref{compare:MNIST} above. In SRT, we set $\lambda = 1$ on both CIFAR-10 and MNIST, and the hyper-parameter of baseline methods are set according to their paper. In particular, there are two default settings discussed in TRADES, we choose the one with better performance ($1/\lambda= 1$ on MNIST and $1/\lambda= 6$ on CIFAR-10). Besides, we observe that when the regularization weight $\beta$ is set to 6, as suggested in the paper, the UAT is hard to learn on the CIFAR-10 dataset. As such, we modify the hyper-parameter $\beta$ to 1 on the CIFAR-10 dataset according to cross-validation. 

\noindent \textbf{Robustness Evaluation Setup}.
We apply FGSM, PGD, momentum iterative attack (MI-FGSM) \cite{dong2018boosting} under $\ell_\infty$ norm in the evaluation, since these methods are similar to the one (PGD) used in solving the inner-maximization in SRT. We also apply some other attack methods, including jacobian saliency map attack (JSMA) \cite{papernot2016limitations}, Carlini \& Wagner attack (C\&W) \cite{carlini2017}, point-wise attack \cite{schott2018}, and decoupled direction and norm attack (DDNA) \cite{rony2019} for the evaluation. All attacks are implemented based on Advertorch \cite{ding2019} and Foolbox \cite{rauber2017}. On the CIFAR-10 dataset, we set perturbation size $\epsilon=0.031$ for FGSM, PGD and MI-FGSM, the step size $\alpha = 0.007$, and the iteration step is set to 10. On the MNIST dataset, we apply PGD-40 and MI-FGSM-40 with step size 0.01 and FGSM under perturbation size 0.3.

\section{Settings for Compound Adversarial Attack and Defense}

\noindent \textbf{Attack Setup}.
There are mainly two different ways to construct a compound attack based on combining different types of attacks. One way is to choose the best attack within many different types of attacks, and the other is to conduct multiple attacks in sequence. The first way is discussed in \cite{tramer2019}, and we discuss the second way in this paper. Let $\bm{\epsilon}=(\epsilon_1, \epsilon_2, \epsilon_3)$ indicates the user-defined maximum perturbation size, where $\epsilon_1, \epsilon_2, \epsilon_3$ is the maximum rotation, maximum translation, and maximum pixel-wise perturbation respectively. We evaluate two types of compound adversarial attacks, including 1) combining GridAdv on the top of PGD (PGD+): $\bm{T}_1(\bm{x})=A(\theta) \cdot \left(\bm{x}+\arg \max_{\bm{r}\in \mathcal{B}_\infty(\epsilon)} \mathcal{L}(f(\bm{x}+\bm{r}), y)\right) + B$ and $dist_1(\bm{T}_1(\bm{x}),\bm{x})=(\theta, \|B\|_{\infty,\infty}) \leq (\epsilon_1, \epsilon_2)$, where $A(\theta) = [\cos{\theta}, -\sin{\theta}; \sin{\theta}, \cos{\theta}]$ and $\mathcal{B}_\infty(\epsilon)=\{\bm{x}|\|\bm{x}\|_\infty \leq \epsilon\}$; and 2) combining PGD on the top of GridAdv (GridAdv+): $\bm{T}_2(\bm{x};\bm{r})=A^{*} \bm{x} + B^{*} +\bm{r}$ and $dist_2(\bm{T}_2(\bm{x}), \bm{x}) = \|\bm{r}\|_\infty \leq \epsilon_3$, where 
$$
(A^{*},B^{*}) = \arg \max_{(A(\theta), B): \theta \leq \epsilon_1, \|B\|_{\infty ,\infty} \leq \epsilon_2} \mathcal{L}(A(\theta)\bm{x}+B,y).
$$

\noindent \textbf{Defense Setup}.
We generate the perturbed samples needed for compound adversarial defense by first performing a pixel-wise attack and then a spatial attack. Specifically, we provide two different compound defenses, including 1) combining Worst-of-$k$ on the top of AT (AT+), and 2) combining spatial SRT on the top of pixel-wise SRT (SRT+).

\newpage

\begin{table}[ht]
\centering
\caption{The computational complexity of all methods discussed in this paper. Note that the hyper-parameter $N, M, I_i$, and $I_o$ may not necessary to be the same across different methods.}
\begin{tabular}{ccc}
\toprule
Standard Training & Worst-of-$k$, KLR, AT, TRADES & UAT, RST, SRT \\ \midrule
$\mathcal{O}\left(N \cdot I_o\right)$ & $\mathcal{O}\left(N \cdot I_i \cdot I_o\right)$ & $\mathcal{O}\left((N+M) \cdot I_i \cdot I_o\right)$   \\ \bottomrule
\end{tabular}
\label{Tab:CC}
\end{table}

\section{Analysis of Computational Complexity}
In this section, we discuss the computational complexity of SRT and all baseline methods.

All baseline defenses are `AT-type', $i.e.$, under the `min-max' framework, and their inner-maximization is solved by PGD. Suppose the number of iterations used in the inner-maximization and the outer-minimization are $I_i$ and $I_o$, respectively. Let $N$, $M$ indicate the number of labeled samples and unlabeled samples, respectively. In general, the computational complexity of all methods is $\mathcal{O}\left((N+M) \cdot I_i \cdot I_o\right)$, while $M=0$ for supervised methods and $I_i=1$ for the standard training. 

In summary, the computational complexity of all methods are shown in Table \ref{Tab:CC}.

\newpage

\begin{figure*}[ht]
\small
\centering
\subfigure[spatial scenario]{
\includegraphics[width=0.47\textwidth]{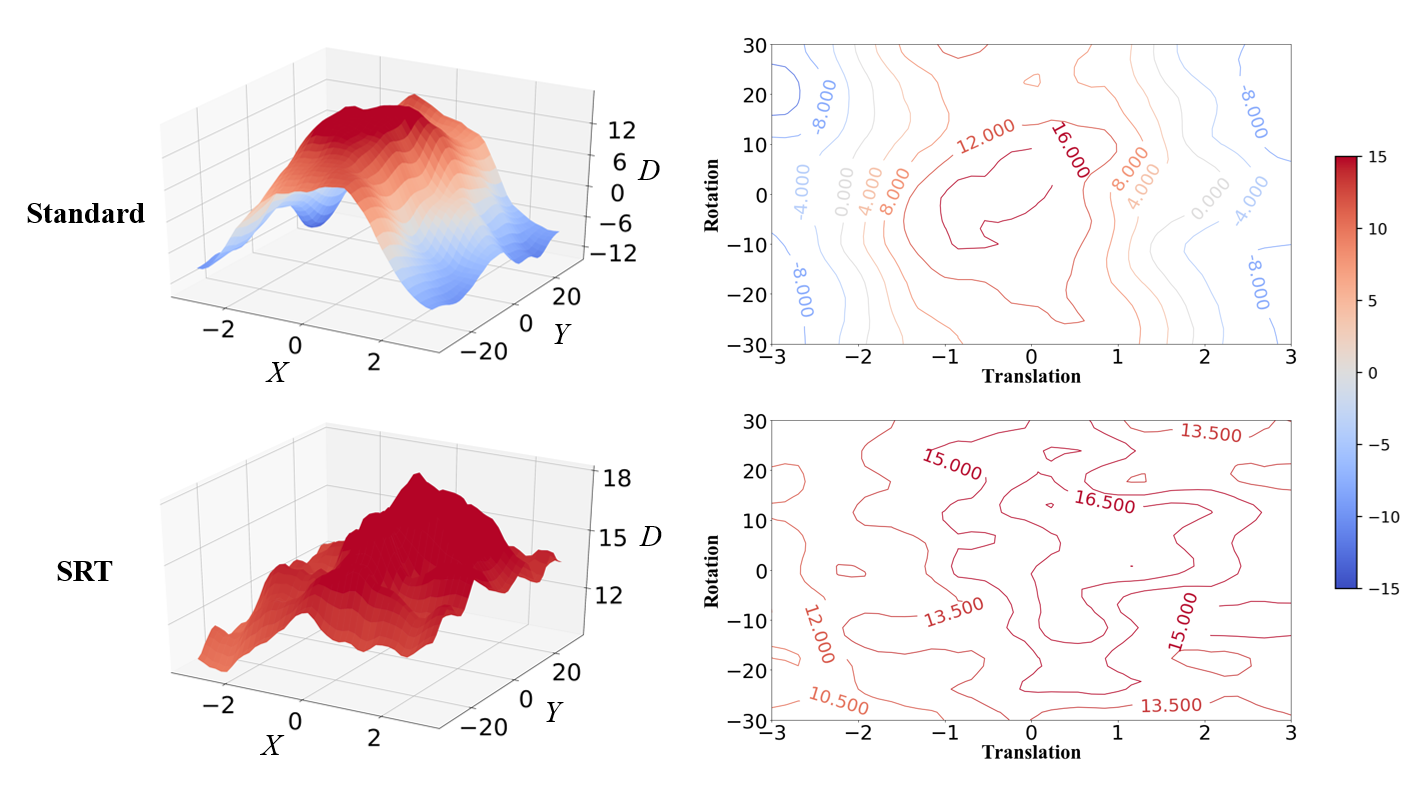}}
\subfigure[pixel-wise scenario]{
\includegraphics[width=0.47\textwidth]{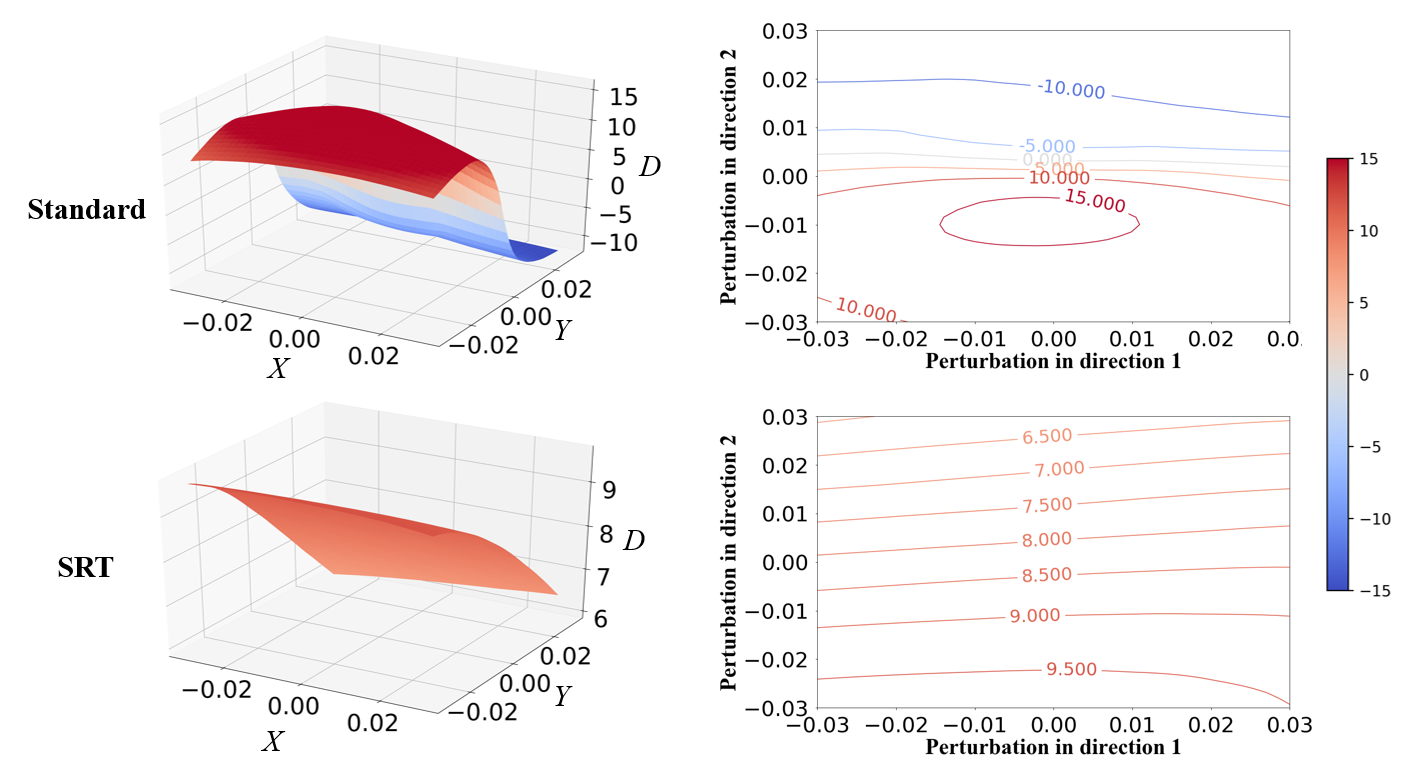}}
\caption{The decision surfaces of SRT-based model and those of the standard training model on the CIFAR-10 dataset. In the 3D decision surfaces, the $X$ and $Y$ axis represent two different perturbation directions and its value indicates the perturbation size, while the $Z$-axis indicates the decision value. If and only if the decision value is positive, the prediction is correct. In those figures, the red area indicates the correctly classified region, while the blue area is the misclassified region.}
\label{ap_fig1}
\end{figure*}

\begin{figure*}[ht]
\small
\centering
\subfigure[spatial scenario]{
\includegraphics[width=0.47\textwidth]{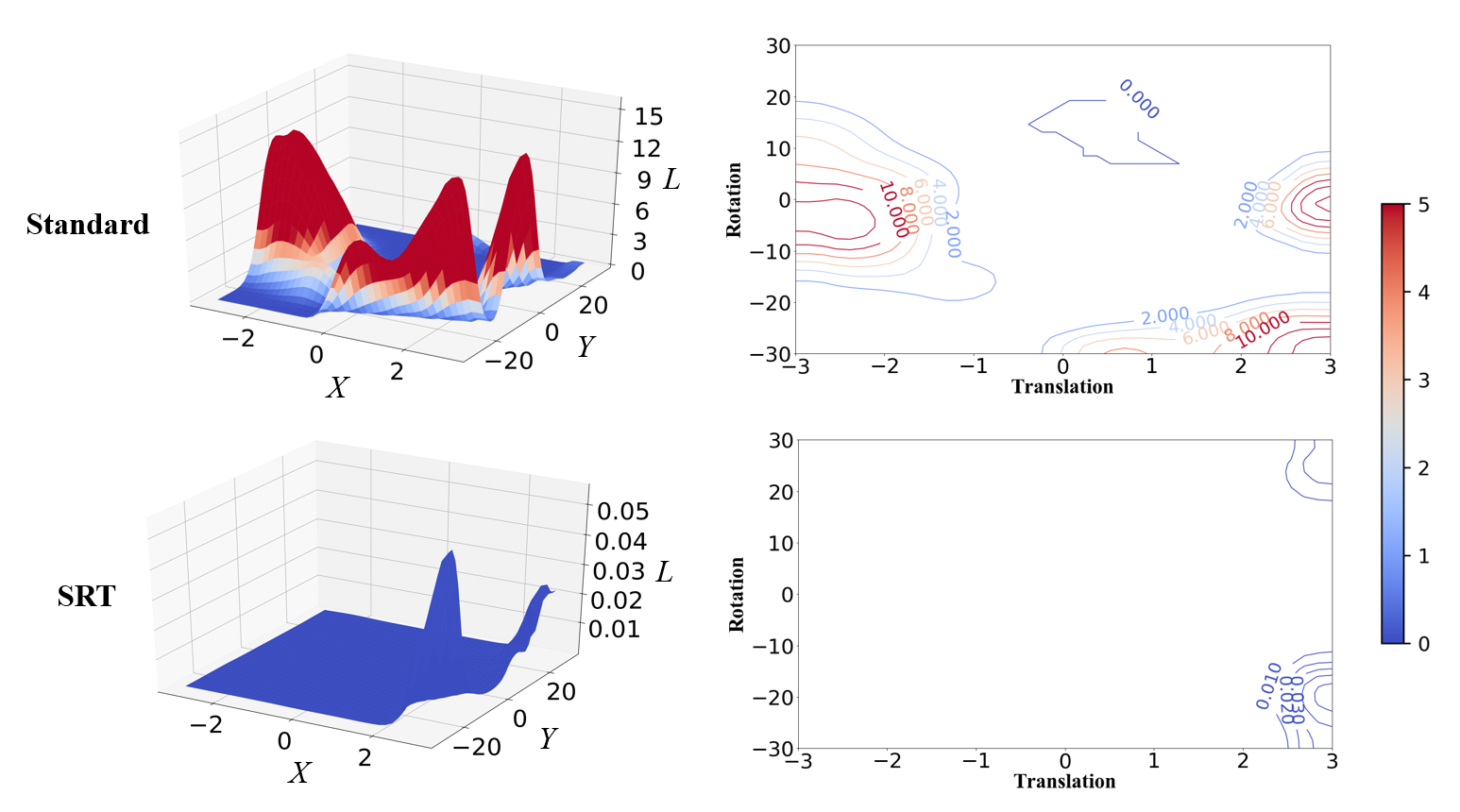}}
\subfigure[pixel-wise scenario]{
\includegraphics[width=0.47\textwidth]{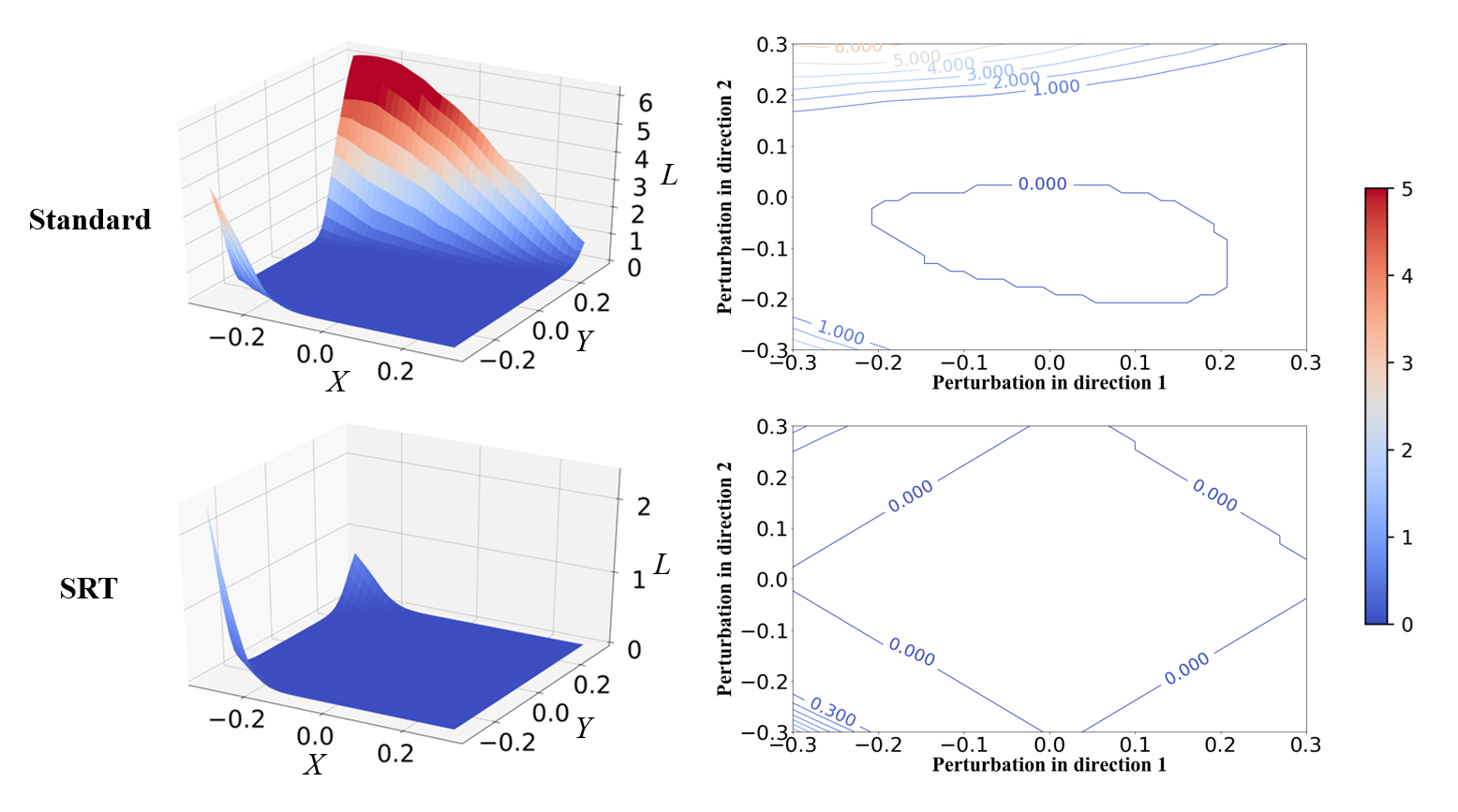}}
\caption{The loss surfaces of SRT-based model and those of the standard training model on the MNIST dataset. In the 3D loss surface, the $X$ and $Y$ axis represent two different perturbation directions and its value indicates the perturbation size, while the $Z$-axis indicates the loss.}
\vskip -0.15in
\label{ap_fig2}
\end{figure*}

\begin{figure*}[!ht]
\small
\centering
\subfigure[spatial scenario]{
\includegraphics[width=0.47\textwidth]{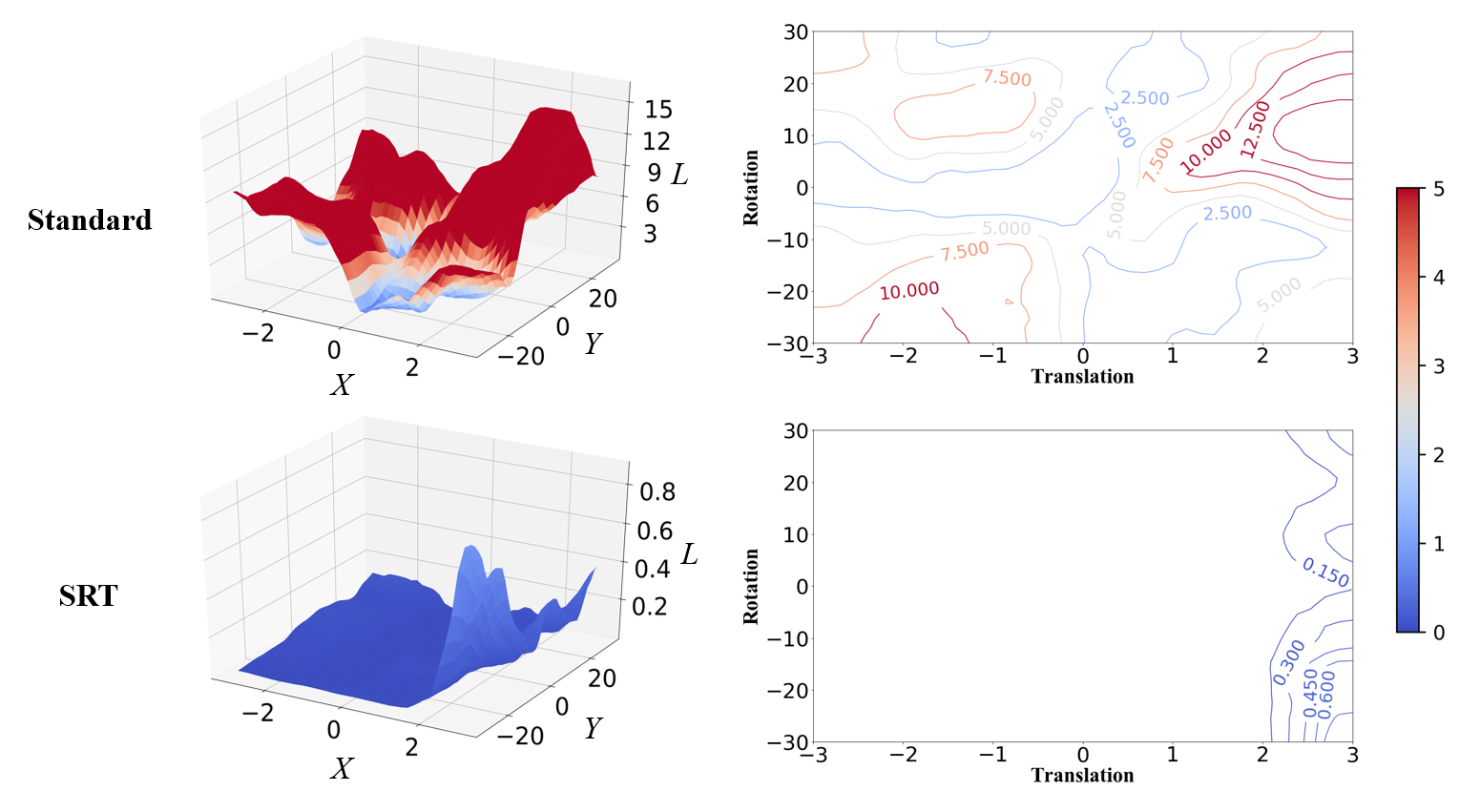}}
\subfigure[pixel-wise scenario]{
\includegraphics[width=0.47\textwidth]{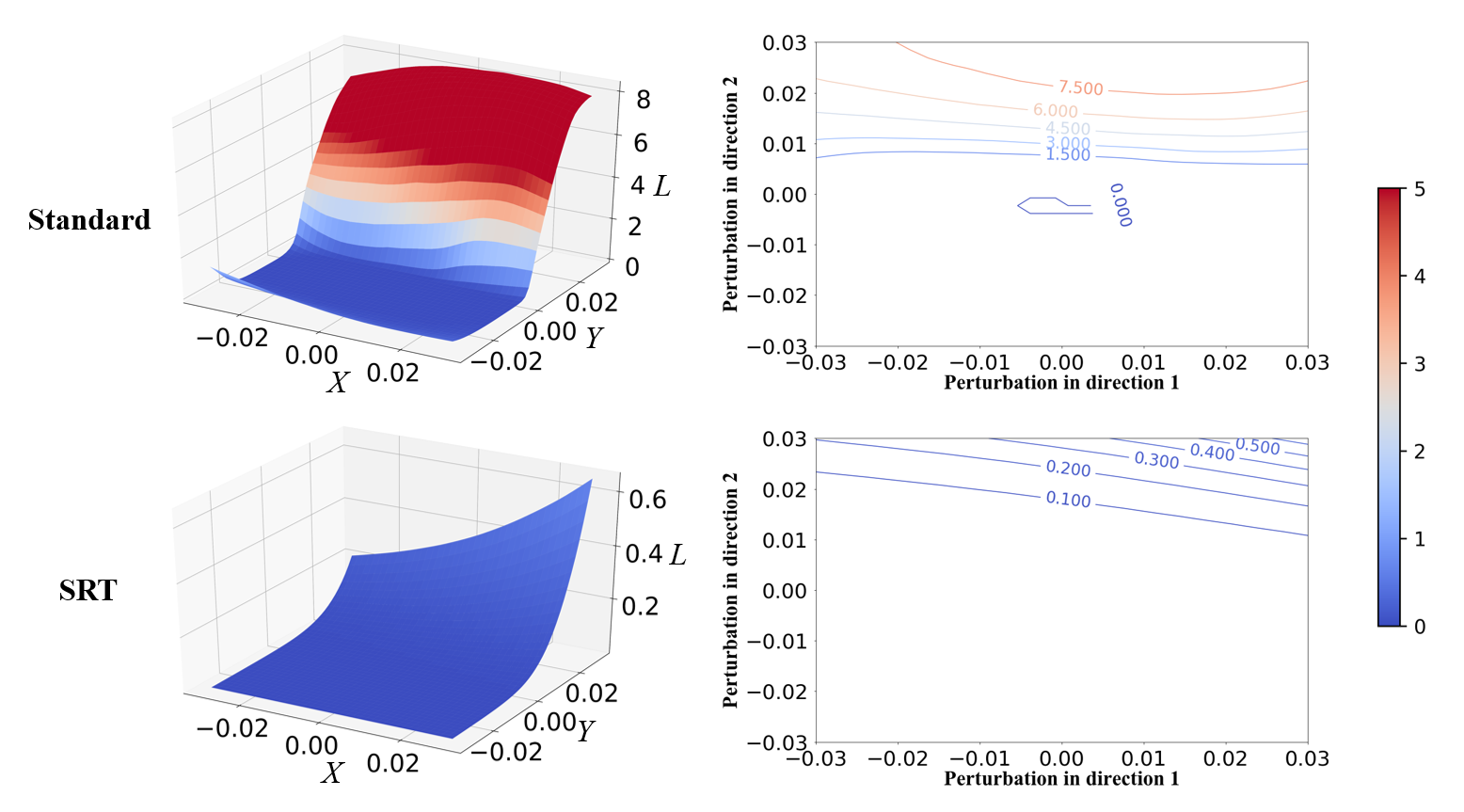}}
\caption{The loss surfaces of SRT-based model and those of the standard training model on the CIFAR-10 dataset. }
\vskip -0.15in
\label{ap_fig3}
\end{figure*}

\section{Additional Decision and Loss Surfaces}

In this section, we present the visualization of decision surface on CIFAR-10 in Figure \ref{ap_fig1}, and the visualizations of loss surfaces on both MNIST \cite{lecun1998} and CIFAR-10 \cite{krizhevsky2009} database in Figures \ref{ap_fig2}-\ref{ap_fig3}.

As shown in Figure \ref{ap_fig1}, the geometric properties of the decision surface on the CIFAR-10 dataset are very similar to those on MNIST. Specifically, compared with SRT, the decision surfaces of the standard training model have sharper peaks and larger slopes, which explains that its prediction is vulnerable to small perturbation. In contrast, the surfaces of SRT are flat and located on a plateau with positive decision value around the benign sample.

\begin{figure}[ht]
\small
\centering
\includegraphics[width=0.7\textwidth]{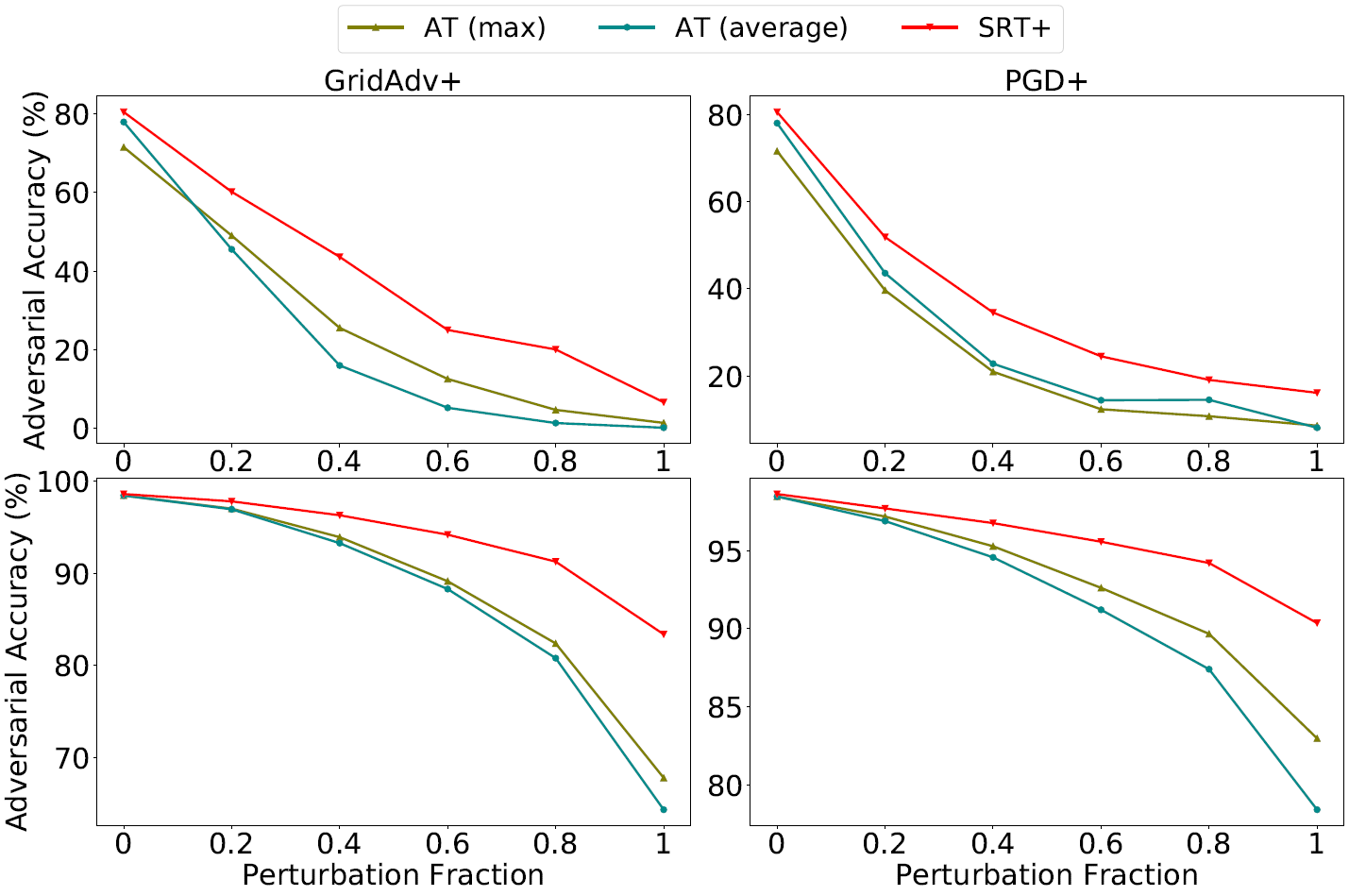}
\caption{The comparison between different types of compound adversarial defenses. The maximum perturbation size is the same as the one used in previous experiments. Perturbation fraction represents the ratio of the current perturbation size to the previous maximum perturbation size. \textbf{First row:} Results on CIFAR-10 dataset. \textbf{Second row:} Results on MNIST dataset.}
\label{ap_fig5}
\vspace{-0.5em}
\end{figure}

Moreover, as shown in Figures \ref{ap_fig2}-\ref{ap_fig3}, the loss surfaces between SRT and those of the model with standard training also have different properties. Specifically, the loss surfaces of the standard training model have much higher peak and larger slopes, compared with the surfaces of SRT. Although the loss value is not directly related to the correctness of prediction compared with the decision value, the aforementioned difference still verifies the effectiveness of SRT.

\section{Different Types of Compound Defenses}

In this section, we compare two different types of the compound adversarial defense, including SRT+ and AT+ proposed in this paper, and those proposed in \cite{tramer2019}, dubbed AT (average) and AT (max). AT (average) and AT (max) respectively minimize the average error rate across perturbation types, or the error rate against an adversary that picks the worst perturbation type for each input. Specifically, for each training sample in the AT (average) and the AT (max), these methods build adversarial examples for all perturbation types and then train either on all examples or only the worst example. We conduct compound adversarial attacks, including GridAdv+ and PGD+, to evaluate the adversarial robustness of the compound adversarial defense methods. All training and evaluation settings are the same as those demonstrated in Section 4.4 of the main manuscript.

As shown in Figure \ref{ap_fig5}, compared with the compound strategies proposed in \cite{tramer2019}, our methods reach better adversarial robustness under compound adversarial attacks. Moreover, in the case of compound defense, our method (SRT+) is also better than all other methods. Especially on the MNIST dataset, compared to the second best method, the AT+, our method achieves an increase of more than $7\%$ in adversarial accuracy under PGD+ attack. This improvement is even more significant under GridAdv+ attack, which is more than $10\%$.

\section{Fast Semi-supervised Robust Training}
Adversarial training based defenses ($e.g.$, AT and TRADES) are often much more time-consuming than standard training. For example, the cost of AT trained in the ImageNet dataset is more than 4,500 GPU hours for ResNet-101, as reported in \cite{xie2019feature}. Such a high cost hinders the use of those methods, especially semi-supervised based methods, which require a large amount of training data. To alleviate this problem, we propose an accelerated version of SRT (dubbed fast SRT) in this section.

\begin{figure*}[!ht]
\centering
\subfigure[spatial scenario]{
\includegraphics[width=0.47\textwidth]{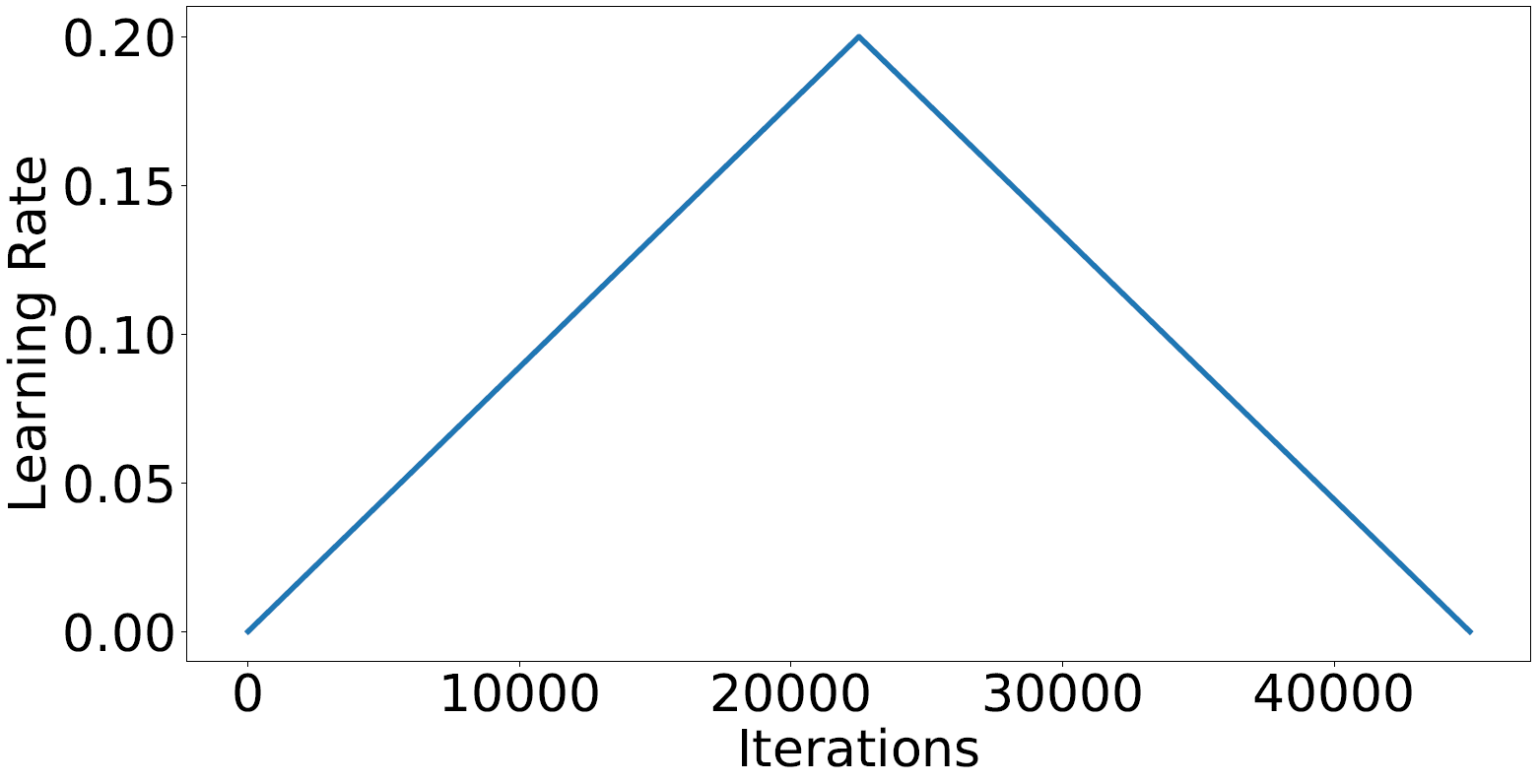}}
\subfigure[pixel-wise scenario]{
\includegraphics[width=0.47\textwidth]{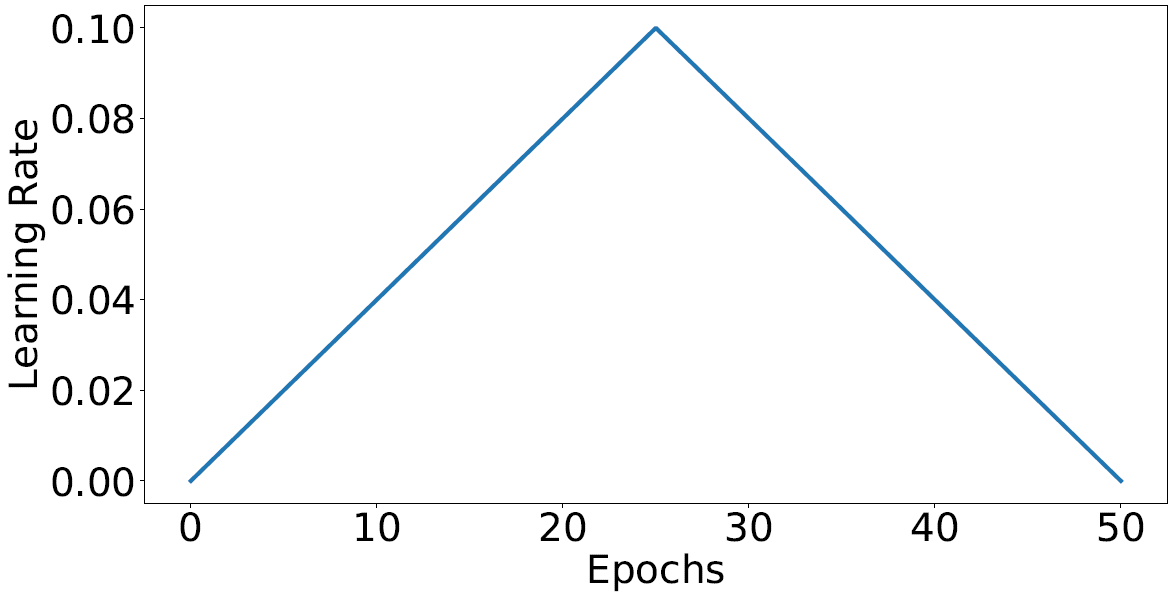}}
\caption{Cyclic learning rates used for fast SRT. }
\label{ap_fig6}
\end{figure*}

\begin{algorithm}[!ht]
   \caption{PGD-based inner-maximization solver with $K$-steps, given benign example $\bm{x}$, loss function $\mathcal{L}$ and step size $\alpha$, maximum perturbation size $\epsilon$ for generating adversarial example $\bm{x}'$.}
   \label{alg: PGD}
\begin{algorithmic}[1]
\STATE $\bm{x}' \leftarrow \bm{x}$ // or initialize with small random Gaussian perturbation
    \FOR{$i=1 \cdots K$}        
        \STATE $\bm{x}' \leftarrow \bm{x}' + \alpha \nabla_{\bm{x}} \mathcal{L}$
        \STATE $\bm{x}' \leftarrow \max{(\min{(\bm{x}', \bm{x} + \epsilon)}, \bm{x}-\epsilon)}$
    \ENDFOR
\STATE $\bm{x}' \leftarrow \max{(\min{(\bm{x}', \bm{1})}, \bm{0})}$
\end{algorithmic}
\end{algorithm}

\begin{algorithm}[!ht]
   \caption{FGSM-based inner-maximization solver, given benign example $\bm{x}$, loss function $\mathcal{L}$ and step size $\alpha$, maximum perturbation size $\epsilon$ for generating adversarial example $\bm{x}'$.}
   \label{alg: FGSM}
\begin{algorithmic}[1]
\STATE $\bm{x}' \leftarrow \bm{x}$ $+$ Uniform($-\epsilon$, $\epsilon$) // initialize with random uniform perturbation
\STATE $\bm{x}' \leftarrow \bm{x}' + \alpha \nabla_{\bm{x}} \mathcal{L}$
\STATE $\bm{x}' \leftarrow \max{(\min{(\bm{x}', \bm{x} + \epsilon)}, \bm{x}-\epsilon)}$
\STATE $\bm{x}' \leftarrow \max{(\min{(\bm{x}', \bm{1})}, \bm{0})}$
\end{algorithmic}
\end{algorithm}

\begin{table}[!ht]
\centering
\caption{Comparison among different models under spatial attacks.}
\begin{tabular}{c|c|ccc|c}
\hline
         & Clean & GridAdv & GridAdv.T & GridAdv.R  & Time (mins)\\ \hline
Worst-of-$k$ & 82.02 & 54.80  & 69.45  & 68.22 & 1,654  \\
KLR          & 85.40 & 56.28  & 72.71  & 71.04 & 1,697  \\ \hline
SRT          & 88.87 & 64.83  & 78.47  & 77.24 & 1,573         \\
fast SRT     & 89.21 & 64.81  & 78.57  & 77.24 & 855   \\ \hline
\end{tabular}
\label{tabel:fastSRT(spatial)}
\vspace{-0.5em}
\end{table}

\begin{table}[!ht]
\centering
\caption{Comparison among different models under pixel-wise attacks.}
\begin{tabular}{c|c|cccccc|c}
\hline
         & Clean & FGSM  & PGD   & MI-FGSM & JSMA  & C\&W  & DDNA  & Time (mins)\\ \hline
AT       & 77.17 & 50.35 & 37.37 & 36.97   & 10.80 & 33.42 & 20.87 & 895              \\
TRADES   & 76.72 & 51.98 & 40.20 & 39.79   & 11.90 & 36.09 & 22.37 & 10,94              \\ \hline
SRT      & 78.46 & 59.34 & 48.66 & 48.24   & 16.99 & 43.33 & 27.71 & 2,916         \\
fast SRT & 82.27 & 58.27 & 44.93 & 44.45   & 12.27 & 40.85 & 26.12 & 257           \\ \hline
\end{tabular}
\label{tabel:fastSRT(pixel-wise)}
\vspace{-0.5em}
\end{table}

In fast SRT, we adopt two acceleration techniques introduced in \cite{wong2020fast}, including \emph{cyclic learning rate} and \emph{FGSM-based inner-maximization solver}. The cyclic schedule can drastically reduce the number of epochs required for training deep networks \cite{smith2017cyclical}, and the FGSM-based solver drastically reduce the iterations for generating adversarial examples (roughly $K$-times faster than the PGD-based solver with $K$-steps). The FGSM-based inner-maximization solver is shown in Algorithm \ref{alg: FGSM}.

\textbf{Settings. } We evaluate the performance of fast SRT on CIFAR-10 dataset. Specifically, the two aforementioned techniques are both involved in the pixel-wise fast SRT, while only cyclic learning rate is adopted in the spatial fast SRT (since its inner-maximization solver is far different from the PGD-based one used in pixel-wise fast SRT). The step size $\alpha$ in FGSM-based solver is set to $1.25\epsilon$ as suggested in \cite{wong2020fast}. We train fast SRT 50 epochs in pixel-wise scenario and 45,000 iterations in spatial scenario, and the specific learning rate schedule is shown in Figure \ref{ap_fig6}. All defense experiments are conducted on one single GeForce GTX 1080 GPU. Other settings are the same as those used in Section 4.2-4.3 in the main manuscript.

\begin{figure*}[!ht]
\centering
\subfigure[spatial scenario]{
\includegraphics[width=0.47\textwidth]{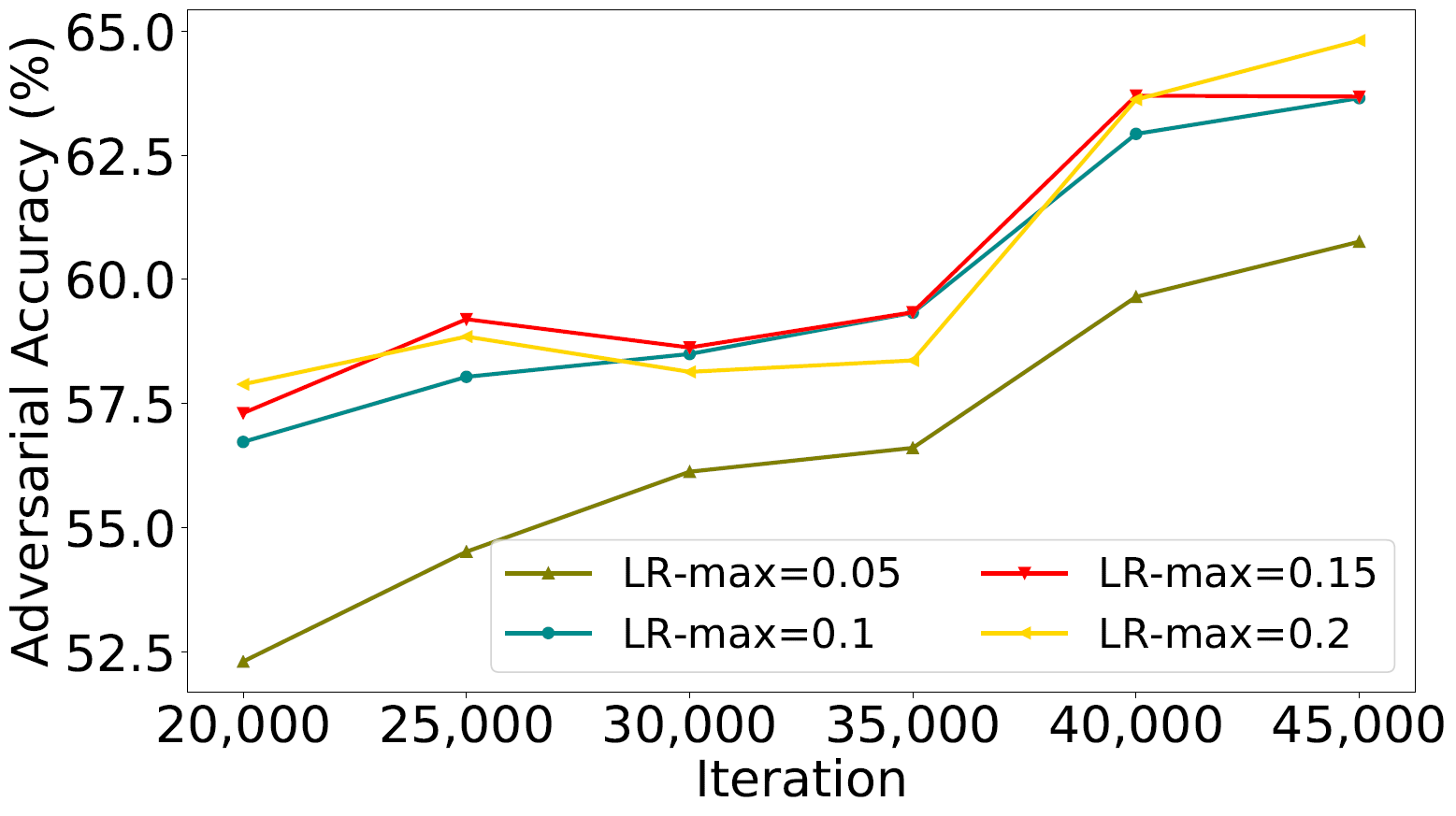}}
\subfigure[pixel-wise scenario]{
\includegraphics[width=0.47\textwidth]{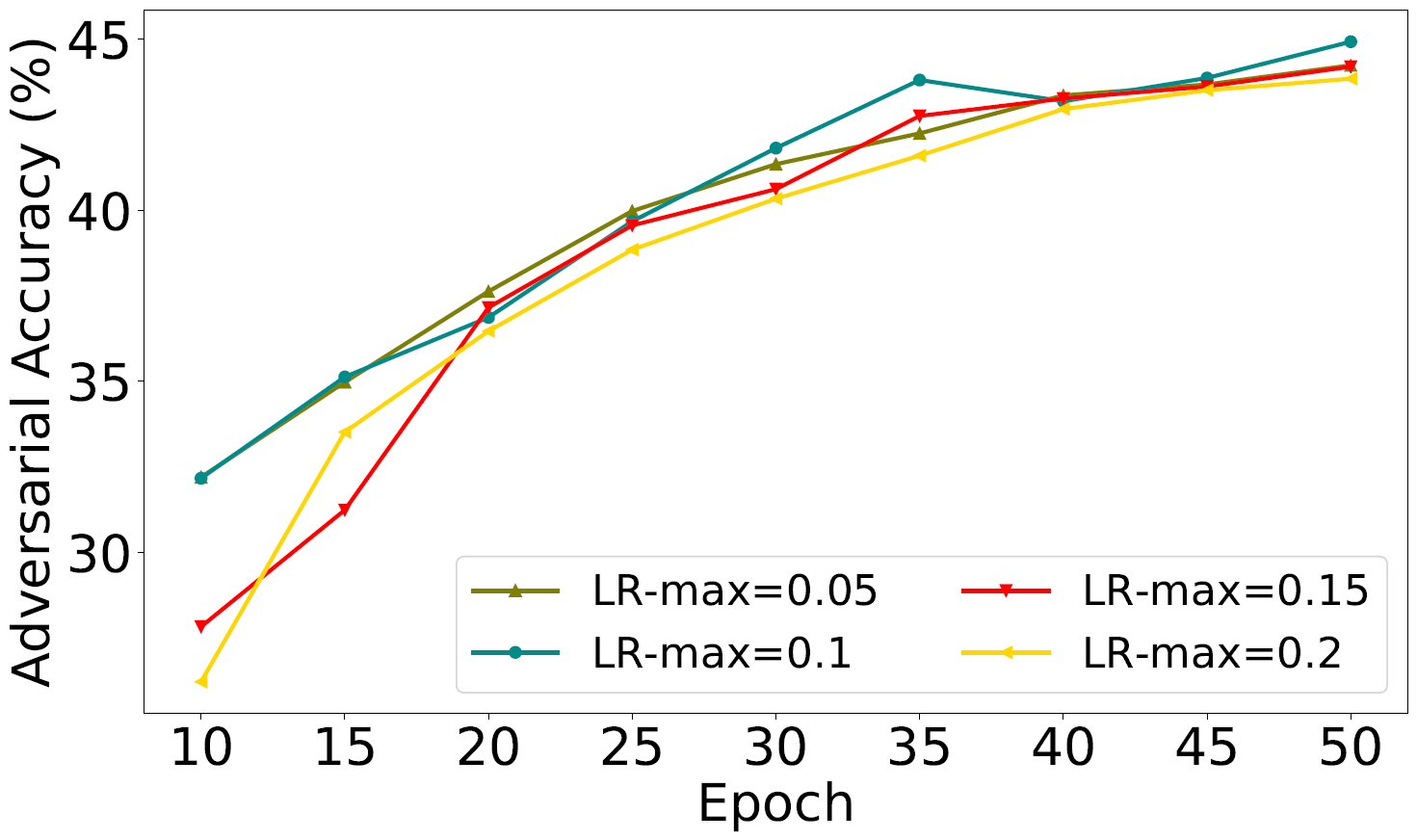}}
\caption{The effect of maximum learning rate and total training time on adversarial robustness. Specifically, we use GridAdv and PGD to evaluate the adversarial robustness in the spatial and pixel-wise scenarios, respectively. The specific attack settings are the same as those used in Section 4.2-4.3 in the main manuscript.}
\label{ap_fig7}
\end{figure*}

\textbf{Results. } As shown in Table \ref{tabel:fastSRT(spatial)}-\ref{tabel:fastSRT(pixel-wise)}, fast SRT is significantly faster than the standard SRT. Especially in the pixel-wise scenario, Fast SRT is ten times faster than the standard version. Although fast SRT has a certain performance degradation compared to SRT especially when FGSM-based solver is used, the performance of fast SRT is still significantly better than supervised defenses ($e.g.$, TRADES, KLR). In other words, fast SRT decreases the computational cost while preserves efficiency.

There are two key hyper-parameters in the fast SRT, including maximum learning rate (dubbed LR-max) and total training time (epochs or iterations). We further visualize their effect on adversarial robustness. As shown in Figure \ref{ap_fig7}, the adversarial robustness increase with the increase of the total training time regardless of the value of LR-max in both spatial and pixel-wise scenarios. Besides, the adversarial accuracy still has an upward trend at the end of the curve, which implies that the robustness can be further improved if additional iterations/epochs are utilized, which brings additional computational costs.

\end{appendices}

\end{document}